\documentclass[lettersize,journal]{IEEEtran}
\usepackage{algorithm}
\usepackage{algpseudocode}
\usepackage{array}
\usepackage[caption=false,font=normalsize,labelfont=sf,textfont=sf]{subfig}
\usepackage{textcomp}
\usepackage{stfloats}
\usepackage{url}
\usepackage{verbatim}
\usepackage{graphicx}
\usepackage{cite}
\usepackage{nccmath}
\usepackage{amsmath}
\usepackage{amssymb}
\usepackage{mathtools}
\usepackage{amsthm}
\usepackage{tabularx}
\usepackage{booktabs}
\usepackage{multirow}

\usepackage[capitalize,noabbrev]{cleveref}

\newtheorem{theorem}{Theorem}

\newtheorem{definition}{Definition}

\newtheorem{assumption}{Assumption}

\newtheorem{lemma}{Lemma}

\newtheorem{corollary}{Corollary}

\newtheorem*{remark}{Remark}

\begin{document}
\sloppy

\title{AQUILA: Communication Efficient Federated Learning with Adaptive Quantization in Device Selection Strategy}

\author{Zihao Zhao$^{*}$,
        Yuzhu Mao$^{*}$,
        Zhenpeng Shi,
        Yang Liu,
        Tian Lan,
        Wenbo Ding\textsuperscript{$\dag$}, and
        Xiao-Ping Zhang
\thanks{$^{*}$ These authors contribute equally.}
\thanks{\textsuperscript{$\dag$} Corresponding author.}
\thanks{Z.~Zhao, Y.~Mao, Z.~Shi, W.~Ding and X.-P. Zhang are with Tsinghua-Berkeley Shenzhen Institute, Tsinghua Shenzhen International Graduate School, Tsinghua University, China. W.~Ding is the corresponding author. E-mail: (\{zhao-zh21, myz20, shizp22\} @mails.tsinghua.edu.cn, ding.wenbo@sz.tsinghua.edu.cn. W.~Ding is also with RISC-V International Open Source Laboratory, Shenzhen, China, 518055.)}
\thanks{Y.~Liu is with the Institute for AI Industry Research (AIR), Tsinghua University, China. E-mail: (liuy03@air.tsinghua.edu.cn). Y.~Liu and W.~Ding are also with Shanghai AI Lab, Shanghai, China.}
\thanks{T.~Lan is with the Department of Electrical and Computer Engineering, George Washington University, DC, USA. Email: (tlan@gwu.edu)}
\thanks{X.-P.~Zhang is also with the Department of Electrical, Computer and Biomedical Engineering, Ryerson University, Toronto, ON M5B 2K3, Canada. E-mail: (xzhang@ee.ryerson.ca)}
}

\markboth{Journal of \LaTeX\ Class Files,~Vol.~14, No.~8, August~2021}%
{Shell \MakeLowercase{\textit{et al.}}: A Sample Article Using IEEEtran.cls for IEEE Journals}

\IEEEpubid{0000--0000/00\$00.00~\copyright~2023 IEEE}

\maketitle

\begin{abstract}
The widespread adoption of Federated Learning (FL), a privacy-preserving distributed learning methodology, has been impeded by the challenge of high communication overheads, typically arising from the transmission of large-scale models. Existing adaptive quantization methods, designed to mitigate these overheads, operate under the impractical assumption of uniform device participation in every training round. Additionally, these methods are limited in their adaptability due to the necessity of manual quantization level selection and often overlook biases inherent in local devices' data, thereby affecting the robustness of the global model. In response, this paper introduces AQUILA (\underline{a}daptive \underline{qu}antization in dev\underline{i}ce se\underline{l}ection str\underline{a}tegy), a novel adaptive framework devised to effectively handle these issues, enhancing the efficiency and robustness of FL. AQUILA integrates a sophisticated device selection method that prioritizes the quality and usefulness of device updates. Utilizing the exact global model stored by devices, it enables a more precise device selection criterion, reduces model deviation, and limits the need for hyperparameter adjustments. Furthermore, AQUILA presents an innovative quantization criterion, optimized to improve communication efficiency while assuring model convergence. Our experiments demonstrate that AQUILA significantly decreases communication costs compared to existing methods, while maintaining comparable model performance across diverse non-homogeneous FL settings, such as Non-IID data and heterogeneous model architectures.
\end{abstract}


\section{Introduction}
\IEEEPARstart{W}{ith} proliferation of ubiquitous sensing and computing devices, the Internet of things (IoT), as well as many other distributed systems, have gradually grown from mere concepts to a reality, bringing dramatic convenience to people's daily lives \cite{vehicularIoT, liu2020fedvision, Googlekeyboard}. In order to fully exploit these distributed computing resources, distributed learning provides a promising framework that parallels the performance of traditional centralized learning schemes. Nevertheless, concerns about the privacy and security of sensitive data during the updating and transmission processes persist. \emph{federated learning (FL)} \cite{mcmahan2017communication}, a methodology developed to address these issues, has been developed, allows distributed devices to collaboratively learn a global model without privacy leakage by keeping private data isolated and masking transmitted information with secure approaches. On account of its potential for privacy-preservation in privacy sensitive fields such as finance and health, FL has garnered substantial from both academia and industry in recent years.

Unfortunately, in many FL applications, such as image classification and objective recognition, the trained model tends to be high-dimensional, resulting in considerable communication costs. Consequently, communication efficiency has emerged as an imperative challenge in FL. In response, \cite{sun2020lazily} proposed the lazily-aggregated quantization (LAQ) method to bypass unnecessary parameter uploads by estimating the gradient innovation: the difference between the current unquantized gradient and the previously quantized gradient. 
Moreover, \cite{mao2021communication} devised an adaptive quantized gradient (AQG) strategy based on LAQ to dynamically select the quantization level within some artificially given numbers during the training process. Nevertheless, AQG has proven insufficiently adaptive due to the difficulty of manually selecting the appropriate quantization levels in complex FL environments.
Alternatively, \cite{jhunjhunwala2021adaptive} introduced an adaptive quantization rule for FL named AdaQuantFL, which searches in a given range for an optimal quantization level and achieves a better error-communication trade-off.

\IEEEpubidadjcol
Existing research on adaptive quantization primarily presumes that all devices in the FL system participate in each training round. However, this assumption is both unrealistic and impracticable. Despite the enormous alleviation in communication overhead via adaptive quantization methods, bandwidth constraints may still be surpassed if all devices transmit their model updates to the server, due to sheer number of devices. Currently, \cite{honig2022dadaquant} proposed a doubly-adaptive quantization algorithm, DAdaQuant, that dynamically adjusts the quantization level across time and devices, and randomly selects $K$ devices per round. Nonetheless, this random sampling provides no theoretical guarantee and could neglect the biases inherent in local devices' data, potentially yielding to underrepresent or overfit to specific patterns and resulting in a less robust global model \cite{cho2020client}. In response to these limitations, this paper introduces a superior adaptive framework, AQUILA, that resorts to a sophisticated device selection method to take the quality and usefulness of the devices' updates into account. Specifically, instead of relying on the estimation of the global gradient such as some existing selection criteria, AQUILA adopts a more precise device selection criterion. This device selection approach uses the exact global model stored by devices and necessitates fewer hyperparameters adjustments. Moreover, we intend to minimize the model deviation induced by the device selection to garner a novel quantization criterion that significantly improves communication efficiency and still offers a convergence guarantee. The contributions of this paper are trifold.



\begin{itemize}
\setlength{\parsep}{0.5ex}
    \item We propose an innovative FL procedure with \textbf{a}daptive \textbf{qu}ant\textbf{i}zation of \textbf{l}azily-\textbf{a}ggregated gradients termed AQUILA, which simultaneously adjusts the communication frequency and the quantization precision in a synergistic fashion.
    
    \item We derive an adaptive quantization strategy from a new perspective that minimizes the model deviation introduced by the device selection. Subsequently, we present a new device selection criterion that is more precise and saves more device storage. Furthermore, we provide a convergence analysis of AQUILA under the generally non-convex case and the Polyak-Łojasiewicz condition.
    
    \item Except for normal FL settings, such as independent and identically distributed (IID) data environment, we experimentally evaluate the performance of AQUILA in a number of \textbf{non-homogeneous} FL settings, such as non-independent and non-identically distributed (Non-IID) local dataset and various heterogeneous model aggregations. The evaluation results reveal that AQUILA considerably mitigates the communication overhead compared to a variety of state-of-art algorithms.
\end{itemize}

\section{Background and Related Works}
Consider an FL system with one central parameter server and a device set $\mathcal{M}$ with $ M=| \mathcal{M} | $ distributed devices to collaboratively train a global model parameterized by $\boldsymbol{\theta}\in \mathbb{R}^{d}$. Each device $m\in\mathcal{M}$ has a private local dataset $\mathcal{D}_{m}=\{(\boldsymbol{x}_{1}^{(m)}, \boldsymbol{y}_{1}^{(m)}), \cdots, (\boldsymbol{x}_{n_m}^{(m)}, \boldsymbol{y}_{n_m}^{(m)})\}$ of $n_m$ samples. The federated training process is typically performed by solving the following optimization problem
\begin{equation}
\begin{aligned}
    \min _{\boldsymbol{\theta}\in \mathbb{R}^{d}} f(\boldsymbol{\theta})&= \frac{1}{M}\sum_{m=1}^{M} f_{m}(\boldsymbol{\theta}) 
    \label{emp_risk}
\end{aligned}
\end{equation}
where $f:\mathbb{R}^{d}\rightarrow\mathbb{R}$ denotes the empirical risk, and $f_{m}:\mathbb{R}^{d}\rightarrow\mathbb{R}$ denotes the local objective based on the private data $\mathcal{D}_m$ of the device $m$.
The FL training process is conducted by iteratively performing local updates and global aggregation as proposed in \cite{mcmahan2017communication}. First, at communication round $k$, each device $m$ receives the global model $\boldsymbol{\theta}^{k}$ from the parameter server and trains it with its local data $D_m$. 
Subsequently, it sends the local gradient $\nabla f_{m}(\boldsymbol{\theta}^{k})$ to the central server, and the server will update the global model with learning rate $\alpha$ by
\begin{align}
\boldsymbol{\theta} ^{k+1} \coloneqq \boldsymbol{\theta} ^{k} - \frac{\alpha}{M}\sum_{m \in \mathcal{M}}\nabla f_{m}(\boldsymbol{\theta}^{k}).
\label{ori_update}
\end{align} 
\begin{definition}
\label{def1}
\textit{For more efficiency, each device only uploads the quantized deflection between the full gradient $\nabla f_m(\boldsymbol{\theta}^{k})$ and the last quantization value $\boldsymbol{q}_{m}^{k-1}$ utilizing a quantization operator $\mathcal{Q}: \mathbb{R}^d \rightarrow \mathbb{R}^d$, i.e.,}
\begin{equation}
    \Delta \boldsymbol{q}_{m}^{k} = \mathcal{Q}(\nabla f_m(\boldsymbol{\theta}^{k})-\boldsymbol{q}_{m}^{k-1}).
\end{equation}
\end{definition} 
For communication frequency reduction, the previous lazy aggregation strategy allows the device $m \in \mathcal{M}$ to upload its newly-quantized gradient innovation at epoch $k$ only when the change in local gradient is sufficiently larger than a threshold. Hence, the quantization of the local gradient $\boldsymbol{q}_{m}^{k}$ of device $m$ at epoch $k$ can be calculated by
\begin{equation}
\boldsymbol{q}_{m}^{k} \coloneqq 
\begin{cases}
\hfil \boldsymbol{q}_{m}^{k-1} , & \text {if } 
\begin{split}
\big\|\mathcal{Q}(\nabla f_m(\boldsymbol{\theta}^{k})-\boldsymbol{q}_{m}^{k-1})\big\|_2^2 \\
\leqslant  Threshold
\end{split}
\\ 
\hfil \boldsymbol{q}_{m}^{k-1} + \Delta \boldsymbol{q}_{m}^{k},
 & \text { otherwise }
 \end{cases}.
\label{lazy_aggre_laq}
\end{equation}
If the device $m$ skips the upload of $\Delta \boldsymbol{q}_{m}^{k}$, the central server will reuse the last gradient $\boldsymbol{q}_{m}^{k-1}$ for aggregation. Therefore, the global aggregation rule can be changed from \eqref{ori_update} to:

\begin{align}
\label{lazy_aggre}
\boldsymbol{\theta} ^{k+1} &= \boldsymbol{\theta} ^{k} - \frac{\alpha}{M}\sum_{m \in \mathcal{M}}\boldsymbol{q}_{m}^{k} \\
&=\boldsymbol{\theta} ^{k}-\frac{\alpha}{M}\sum_{m \in \mathcal{M}^k}\left(\boldsymbol{q}_{m}^{k-1} + \Delta \boldsymbol{q}_{m}^{k}\right)- 
\frac{\alpha}{M} \sum_{m \in \mathcal{M}_{c}^{k}}\boldsymbol{q}_{m}^{k-1},\nonumber
\end{align}
where $\mathcal{M}^{k}$ denotes the subset of devices that upload their quantized gradient innovation, and $\mathcal{M}_c^k = \mathcal{M}\setminus \mathcal{M}^{k}$ denotes the subset of devices that skip the gradient update and reuse the old quantized gradient at epoch $k$. 

For AdaQuantFL, it is proposed to achieve a better error-communication trade-off by adaptively adjusting the quantization levels during the FL training process. Specifically, AdaQuantFL computes the optimal quantization level $(b^{k})^{*}$ by $(b^{k})^{*} =\lfloor\sqrt{{f(\boldsymbol{\theta}^{0})}/{f(\boldsymbol{\theta}^{k})}}\cdot b_{0}\rfloor$, where $f(\boldsymbol{\theta}^{0})$ and $f(\boldsymbol{\theta}^{k})$ are the global objective loss defined in \eqref{emp_risk}.

However, AdaQuantFL transmits quantized gradients \textbf{for all local devices at  every communication round}. In order to skip unnecessary communication rounds and adaptively adjust the quantization level for each communication jointly, a naive approach is to quantize lazily aggregated gradients with AdaQuantFL. Nevertheless, it fails to achieve efficient communication for several reasons. First, given the descending trend of training loss, AdaQuantFL's criterion may lead to a high quantization bit number even exceeding 32 bits in the training process (assuming a floating point is represented by 32 bits in our case), which is too large for cases where the global convergence is already approaching and makes the quantization meaningless. Second, a higher quantization level results in a smaller quantization error, leading to a lower communication threshold in the lazy aggregation criterion \eqref{lazy_aggre_laq} and thus a higher transmission frequency.

Consequently, it is desirable to develop a more efficient adaptive quantization method in the device selection setting to improve communication efficiency in FL systematically.

\begin{table*}[htp]
    \centering
    \caption{Overview of previous work on adaptive quantization strategies and comparsion with AQUILA.}
    \label{theoretic_related_work}
    \resizebox{1\textwidth}{!}{%
    \begin{tabular}{c|c|c|c|c|c|c}
        \toprule
        \multirow{2}{*}{\begin{tabular}{@{}c@{}}Adaptive\\ method\end{tabular}} & \multirow{2}{*}{\begin{tabular}{@{}c@{}}Optimization\\ objective\end{tabular}} & \multirow{2}{*}{\begin{tabular}{@{}c@{}}Additional\\ constrain\end{tabular}} & \multirow{2}{*}{\begin{tabular}{@{}c@{}}Convergence\\ guarantee\end{tabular}} & \multirow{2}{*}{\begin{tabular}{@{}c@{}}Non-IID\\ devices\end{tabular}} & \multirow{2}{*}{\begin{tabular}{@{}c@{}}Heterogeneous\\ model\end{tabular}} & \multirow{2}{*}{\begin{tabular}{@{}c@{}}Both text and \\ vision datasets\end{tabular}} \\
        &&&&&&\\
        \midrule
        AdaQuantFL~\cite{jhunjhunwala2021adaptive} & Convergence upper bound & Null & \checkmark & \checkmark & $\times$ & $\times$ \\
        \midrule
        FedDQ~\cite{qu2022feddq} & Convergence upper bound & Total communication costs & \checkmark & $\times$ & $\times$ & $\times$ \\
        \midrule
        Lin et al.~\cite{lin2021channel} & \begin{tabular}{@{}c@{}}SNR of the channel noise \\ \& the quantization noise.\end{tabular} & Total communication costs & \checkmark & $\times$ & $\times$ & $\times$ \\
        \midrule
        AdaGQ~\cite{liu2023communication} & \begin{tabular}{@{}c@{}}Global loss \&\\ wall-clock training time\end{tabular} & Total communication costs & $\times$ & \checkmark & $\times$ & $\times$ \\
        \midrule
        AQeD~\cite{liu2022ensemble} & Global loss & \begin{tabular}{@{}c@{}}Total wireless bandwidth \&\\ total KL-divergence\end{tabular} & \checkmark & \checkmark & \checkmark & $\times$ \\
        \midrule
        DAdaQuant~\cite{honig2022dadaquant} & Total communication costs & Quantization error & $\times$ & \checkmark & $\times$ & $\times$ \\
        \midrule
        \textbf{AQUILA (ours)} & Model deviation & Null & \checkmark & \checkmark & \checkmark & \checkmark \\
        \bottomrule
    \end{tabular}%
    }
\end{table*}
\textbf{Related works on adaptive quantization algorithms.}
Numerous studies have delved into adaptive quantization within FL. For one thing, From a heuristic viewpoint, some research acknowledges the varied communication bandwidths among heterogeneous edge devices in FL.
For instance, Qu et al.~\cite{qu2020quantization} introduce an adaptive quantization strategy that sets the quantization level in proportion to a device's local communication bandwidth.
Meanwhile, CDAG-FL\cite{li2023adaptive} differentiates quantization levels for individual model layers, leveraging the K-Means algorithm for selection. Sun et al.~\cite{sun2020adaptive} establish the adaptive quantization level considering the gradient's total bit length and a predefined maximum throughput, albeit with the inclusion of extra parameters.

In contrast, other research ventures into adaptive quantization from a theoretical view. The primary distinction among these studies is the methodology employed to formulate the optimization problem with respect to the quantization level.
One notable group focuses on convergence analysis. For example, AdaQuantFL~\cite{jhunjhunwala2021adaptive} establishes an error upper bound for the expected loss function and minimizes this bound in relation to the quantization level, pinpointing the optimal level. This method, however, yields a rising trend in quantization levels, consequently increasing communication overheads. To counteract this, FedDQ~\cite{qu2022feddq} optimizes the convergence upper bound, incorporating communication volume constraints.
Beyond convergence-bound optimization, Lin et al.\cite{lin2021channel} endeavor to optimize the signal-to-noise ratio (SNR), considering channel noise, quantization noise, and an overarching quantization level constraint. AdaGQ\cite{liu2023communication} focuses on shortening wall-clock training time, while AQeD~\cite{liu2022ensemble} roots its approach in clustering, categorizing devices into clusters with similar local models and diverse quantization levels. Their augmented loss function uniquely combines ensemble distillation loss, quantization levels, and wireless resource limitations.
Furthermore, a particularly pertinent work, DAdaQuant~\cite{honig2022dadaquant}, introduces a doubly-adaptive quantization algorithm that adjusts quantization levels both temporally and across devices, selecting $K$ devices per iteration. 
Nevertheless, in comparison to our method, their random sampling lacks a solid theoretical underpinning, potentially resulting in biases and, subsequently, the global model's diminished robustness~\cite{cho2020client}. \Cref{theoretic_related_work} provides an overview of these theoretical works and highlights our contributions: 1) We introduce a fresh perspective on determining optimal quantization by minimizing model deviation due to devices skipping; 2) Our objective function is free from additional constraints; 3) We establish a convergence assurance for AQUILA and demonstrate its efficacy across diverse FL scenarios.

\section{Adaptive Quantization in Device Selection Strategy}
Given the above limitations of the naive joint use of the existing adaptive quantization criterion and device selection strategy, this paper aims to design a unifying procedure for communication efficiency optimization where the quantization level and communication frequency are considered synergistically and interactively.

\subsection{Precise device selection criterion}
\label{precise_skip}

First, we introduce the definition of a deterministic rounding quantizer and its corresponding quantization error.

\begin{definition}
(Deterministic mid-tread quantizer). \textit{Every element of the gradient innovation of device $m$ at epoch $k$ is mapped to an integer $[{\boldsymbol{\psi}}_{m}^{k}]_i$ as $\forall i \in \{1,2,...,d\}$ }
\begin{equation}
\begin{aligned}
\label{psi_q}
\left[{\boldsymbol{\psi}}_{m}^{k}\right]_i=&\left\lfloor\frac{\left[\nabla f_m(\boldsymbol{\theta}^k)\right]_i-\left[\boldsymbol{q}_{m}^{k-1}\right]_i+R_m^k}{2 \tau_m^k R_m^k}\!+\!\frac{1}{2}\right\rfloor,
\end{aligned}
\end{equation}
where $\nabla f(\boldsymbol{\theta}_m^k)$ denotes the current unquantized gradient, $R_m^k = \|\nabla f_m(\boldsymbol{\theta}^{k})-\boldsymbol{q}_{m}^{k-1}\|_\infty$ denotes the quantization range, $b_m^k$ denotes the quantization level, and $\tau_m^k:=1 /(2^{b_m^k}-1)$ denotes the quantization granularity.
\end{definition}

\begin{figure}[htb]
    \centering
    \includegraphics[width=0.8\linewidth]{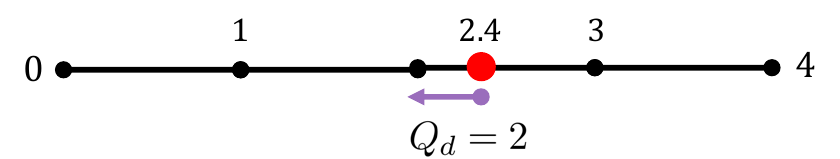}
    \caption{An example of deterministic mid-tread quantizer. In this figure, suppose the step-size $\Omega = 1$ and the original value of $v_i$ is equal to $2.4$. According to the quantizer, $Q_d(v_i)$ will be mapped to $\left\lfloor v_i \right\rfloor = 2$.}
    \label{quant_process}
\end{figure}
For the intuitions of the quantization operator, we can consider a simpler version of the above quantizer: 
$$
Q_d\left(v_i\right)=\left\lfloor v_i / \Omega\right\rfloor * \Omega.
$$
where $Q_d(\cdot)$ denotes the quantization operator (a.k.a, the quantizer) with the quantization level $d$, $\boldsymbol{v}$ denotes the vector needed to be quantized, and $\Omega$ denotes the step-size of the deterministic quantization. \Cref{quant_process} illustrates the quantization process.

The choice of a deterministic quantizer is primarily driven by its computational efficiency, especially for low-resource devices in FL scenarios. In contrast, stochastic quantization methods introduce overheads due to the generation of random numbers for each weight update. Consequently, they have not been widely adopted in practice, as highlighted in \cite{gholami2022survey}. Moreover, in situations where precise weight updates are imperative, such as in fine-tuning pre-trained models, the inherent noise from stochastic quantization might result in divergence or sub-optimal convergence.
\begin{definition}(Quantization error). \textit{
The global quantization error $\boldsymbol{\varepsilon}^k$ is defined by the subtraction between the current unquantized gradient $\nabla f (\boldsymbol{\theta}^k)$ and its quantized value $\boldsymbol{q}^{k-1} + \Delta \boldsymbol{q}^k$, i.e.,}
\begin{equation}
    \boldsymbol{\varepsilon}^k = \nabla f (\boldsymbol{\theta}^k) - \boldsymbol{q}^{k-1} - \Delta \boldsymbol{q}^k,
    \label{def_quanti_error}
\end{equation}
where the current global terms can be computed as $\nabla f (\boldsymbol{\theta}^k) \!=\! \sum_{m \in \mathcal{M}} \nabla f_m (\boldsymbol{\theta}^k), \Delta 
\boldsymbol{q}^{k} \!=\! \sum_{m \in \mathcal{M}} \Delta \boldsymbol{q}_m^{k},   \boldsymbol{q}^{k-1} \!=\! \sum_{m \in \mathcal{M}} \boldsymbol{q}_m^{k-1} .$
\end{definition}

In AQUILA, we propose a novel communication criterion aimed at preventing the unintentional oversight of device group expansions: for $m \in \mathcal{M}_c^k$, the device $m$ will skip its model transmission to the server at epoch $k$ if the following inequality is satisfied:
\begin{equation}
\begin{aligned}
 \left\| \Delta \boldsymbol{q}_{m}^{k}\right\|_2^2 + \left\|\boldsymbol{\varepsilon}_m^{k}\right\|_2^2 &\leqslant \frac{\beta}{ \alpha^2}\left\| \boldsymbol{\theta}^{k}-\boldsymbol{\theta}^{k-1} \right\|_2^2,
\end{aligned}
\label{skip_rule}
\end{equation}
where $\beta \geqslant 0$ is a tuning factor. Note that this skipping rule is employed at epoch $k$, in which each device $m$ calculates its quantized gradient innovation $\Delta \boldsymbol{q}_{m}^{k}$ and quantization error $\boldsymbol{\varepsilon}_m^{k}$, then utilizes this rule to decide whether uploads $\Delta \boldsymbol{q}_{m}^{k}$.

Instead of storing a large number of previous model parameters as LAQ, the strength of \eqref{skip_rule} is that AQUILA directly utilizes the global model for two adjacent rounds as the skip condition, which does not need to estimate the global gradient (more precise), requires fewer hyperparameters to adjust, and considerably reduces the storage pressure of local devices. This is especially important for small-capacity devices (e.g., sensors) in practical IoT scenarios. Furthermore, with the given threshold, AQUILA has a good theoretical property. The theoretical analysis of AQUILA is easier to follow with no Lyapunov function introduced as in LAQ. And the result in \ref{development} also shows that AQUILA can achieve a better convergence rate under the non-convex case and the PL condition.

\subsection{Optimal quantization level}
As mentioned before, AQUILA intends to minimize the model deviation induced by the device selection to deduce how each local device chooses the optimal quantization level. First, we introduce the definition of the fully-aggregated model.
\begin{definition}(Fully-aggregated model). 
\textit{The fully-aggregated model $\boldsymbol{\Tilde{{\theta}}}$ \textbf{without} device skipping at epoch $k$ is computed by}
\begin{equation}
\boldsymbol{\Tilde{{\theta}}} ^{k+1} = \boldsymbol{\theta} ^{k}- \frac{\alpha}{M}\sum_{m \in \mathcal{M}}\left(\boldsymbol{q}_{m}^{k-1}+ \Delta \boldsymbol{q}_{m}^{k} \right).
\label{appendix1}
\end{equation}
\end{definition}
\begin{lemma}
    \textit{The influence of device skipping at communication round $k$ can be bounded by}
\begin{align}
\left\|\Tilde{\boldsymbol{\theta}}^k\!-\!\boldsymbol{\theta}^k\right\|_2^2 \!\leqslant&\!\frac{4\alpha^2|\mathcal{M}_c^k| }{M^2}\!\sum_{m \in \mathcal{M}_c^{k}} \Big( \big( \left\|\nabla f_m(\boldsymbol{\theta}^k)\!-\!\boldsymbol{q}_{m}^{k-1}\right\|_2\! \nonumber
\\
&-\!\left\|\tau_m^k R_m^k \boldsymbol{1}\right\|_2 \big)^2\!+\!4(R_m^k)^2 d\!+\!\frac{d}{2} \Big).
\end{align}
\label{lemma3.3}
\end{lemma}
\begin{proof}
To prove this Lemma, we will use the following equality and inequalities.
Suppose $n \in \mathbb{N}^+$ and $\| \cdot \|_2$ denotes the $\ell^2-$norm. For $p$ in $\mathbb{R}^{+}, \boldsymbol{x}_i, \boldsymbol{a}, \boldsymbol{b} \in \mathbb{R}^d$, there holds:
\begin{enumerate}
    \item Inner product equality.
    \begin{equation}
    \langle\boldsymbol{a}, \boldsymbol{b}\rangle=\frac{1}{2}\left(\|\boldsymbol{a}\|_2^2+\|\boldsymbol{b}\|_2^2-\|\boldsymbol{a}-\boldsymbol{b}\|_2^2\right).
    \label{inner_pro}
    \end{equation}
    \item Norm-summation inequality.
    \begin{equation}
        \left\|\sum_{i = 1}^{n} \boldsymbol{x}_i \right\|_2^2 \leqslant n \sum_{i=1}^{n} \left\| \boldsymbol{x}_i\right\|_2^2.
        \label{sum_ineq}
    \end{equation}
    \item Young's inequality.
    \begin{equation}
        \left\| \boldsymbol{a} + \boldsymbol{b} \right\|_2^2 \leqslant (1 + p)\left\| \boldsymbol{a} \right\|_2^2 + (1 + p^{-1})\left\| \boldsymbol{b} \right\|_2^2.
        \label{youngs}
    \end{equation}
    
    \item Minkowski’s inequality.
    \begin{equation}
        \left\| \boldsymbol{a} + \boldsymbol{b} \right\|_2 \leqslant \left\| \boldsymbol{a} \right\|_2 + \left\| \boldsymbol{b} \right\|_2.
        \label{Minko}
    \end{equation}
\end{enumerate}

With device selection, the aggregated model at epoch $k$ is:
\begin{equation}
\boldsymbol{\theta} ^{k+1} = \boldsymbol{\theta} ^{k}- \frac{\alpha}{M}\sum_{m \in \mathcal{M}^k }\left(\boldsymbol{q}_{m}^{k-1}+ \Delta \boldsymbol{q}_{m}^{k} \right)- \frac{\alpha}{M} \sum_{m \in \mathcal{M}_{c}^{k}}\boldsymbol{q}_{m}^{k-1}.
\label{appendix2}
\end{equation}

Suppose $\Delta_m^k$ denotes the \textbf{quantization loss} of device $m$ at epoch $k$ and $\boldsymbol{\psi}_{m}^{k}$ denotes the quantization representation of local gradient innovation as in \cref{def1}, i.e.,
\begin{equation}
\Delta_m^k=
\boldsymbol{\psi}_{m}^{k}-\frac{\nabla f_m\left(\boldsymbol{\theta}^k\right)-\boldsymbol{q}_{m}^{k-1}+R_m^k\boldsymbol{1}}{2 \tau_m^k R_m^k}-\frac{1}{2}\boldsymbol{1}
\label{def_Delta}
\end{equation}

With \eqref{appendix1}, \eqref{appendix2}, and \eqref{def_Delta}, and for mathematical simplicity we denote $\Gamma = \cfrac{\alpha^2|\mathcal{M}_c^k|}{M^2}$, the model deviation $\|\Tilde{\boldsymbol{\theta}}^k-\boldsymbol{\theta}^k\|_2^2$ caused by skipping gradients can be written as:
\begin{equation}
\scriptsize
\thinmuskip=0mu
\medmuskip=0mu
\thickmuskip=0mu
\begin{aligned}
&\phantom{=} \left\|\Tilde{\boldsymbol{\theta}}^k-\boldsymbol{\theta}^k\right\|_2^2\\
&=\bigg\|\frac{\alpha}{M}\sum_{m \in \mathcal{M}_c^{k}}\left(2 \tau_m^k R_m^k \boldsymbol{\psi}_{m}^{k}-R_m^k \boldsymbol{1}\right)\bigg\|_2^2\\
&\overset{\eqref{sum_ineq}}{\leqslant} \Gamma \sum_{m \in \mathcal{M}_c^{k}}\left\|2 \tau_m^k R_m^k \boldsymbol{\psi}_{m}^{k}-R_m^k \boldsymbol{1}\right\|_2^2\\
& \overset{\eqref{def_Delta}}{=} \Gamma \sum_{m \in \mathcal{M}_c^{k}} \left(\left\|\nabla f_m(\boldsymbol{\theta}^k)-\boldsymbol{q}_{m}^{k-1}+R_m^k \boldsymbol{1}+\tau_m^k R_m^k \boldsymbol{1}+2 \tau_m^k R_m^k\Delta_m^k - R_m^k \boldsymbol{1}\right\|_2^2\right)\\
& \overset{\eqref{sum_ineq}}{\leqslant} 2\Gamma \sum_{m \in \mathcal{M}_c^{k}} \left(\left\|\nabla f_m(\boldsymbol{\theta}^k)-\boldsymbol{q}_{m}^{k-1}+\tau_m^k R_m^k \boldsymbol{1}\right\|_2^2 + \left\|2 \tau_m^k R_m^k \Delta_m^k\right\|_2^2\right)\\
& \overset{(a)}{\leqslant} 2\Gamma \sum_{m \in \mathcal{M}_c^{k}} \left(\left\|\nabla f_m(\boldsymbol{\theta}^k)-\boldsymbol{q}_{m}^{k-1}+\tau_m^k R_m^k \boldsymbol{1}\right\|_2^2 + 4 (\tau_m^k R_m^k)^2 d\right)\\
& \overset{\eqref{Minko}}{\leqslant} 2\Gamma \sum_{m \in \mathcal{M}_c^{k}} \left(\left(\left\|\nabla f_m(\boldsymbol{\theta}^k)-\boldsymbol{q}_{m}^{k-1}\right\|_2+\left\|\tau_m^k R_m^k \boldsymbol{1}\right\|_2\right)^2 + 4 (\tau_m^k R_m^k)^2 d\right)\\
& = 2\Gamma \sum_{m \in \mathcal{M}_c^{k}} \bigg(\left(\left\|\nabla f_m(\boldsymbol{\theta}^k)-\boldsymbol{q}_{m}^{k-1}\right\|_2-\left\|\tau_m^k R_m^k \boldsymbol{1}\right\|_2+2\left\|\tau_m^k R_m^k \boldsymbol{1}\right\|_2\right)^2 + \\ 
&\phantom{ = 2\Gamma \sum_{m \in \mathcal{M}_c^{k}} \bigg((\left\|\nabla f_m(\boldsymbol{\theta}^k)-\boldsymbol{q}_{m}^{k-1}\right\|_2-\left\|\tau_m^k R_m^k \boldsymbol{1}\right\|_2}  \qquad \qquad \quad 4 (\tau_m^k R_m^k)^2 d\bigg)\\
& \leqslant 4\Gamma \sum_{m \in \mathcal{M}_c^{k}} \bigg(\left(\left\|\nabla f_m(\boldsymbol{\theta}^k)-\boldsymbol{q}_{m}^{k-1}\right\|_2-\left\|\tau_m^k R_m^k \boldsymbol{1}\right\|_2\right)^2+4\left\|\tau_m^k R_m^k \boldsymbol{1}\right\|_2^2 + \\ 
&\phantom{ = 2\Gamma \sum_{m \in \mathcal{M}_c^{k}} \bigg((\left\|\nabla f_m(\boldsymbol{\theta}^k)-\boldsymbol{q}_{m}^{k-1}\right\|_2-\left\|\tau_m^k R_m^k \boldsymbol{1}\right\|_2}  \qquad \qquad \quad 2 (\tau_m^k R_m^k)^2 d\bigg)\\
& \overset{(b)}{\leqslant} 4\Gamma \sum_{m \in \mathcal{M}_c^{k}} \left(\left(\left\|\nabla f_m(\boldsymbol{\theta}^k)-\boldsymbol{q}_{m}^{k-1}\right\|_2-\left\|\tau_m^k R_m^k \boldsymbol{1}\right\|_2\right)^2+6 (R_m^k)^2 d\right),
\label{model_deviation}
\end{aligned}
\end{equation}
where $\boldsymbol{1} \in \mathbb{R}^d$ denotes the vector filled with scalar value $1$, (a) $[\Delta_m^k]_i \in (-1, 0]$, (b) $R_m^k \geqslant \tau_m^k R_m^k \geqslant 0$.
\end{proof}
Corresponding to \cref{lemma3.3}, since $R_m^k$ is independent of $\tau_m^k$, we can formulate an optimization problem to minimize the upper bound of this model deviation caused by update skipping for each device $m$:
\begin{equation}
   \begin{array}{cl}
\underset{0<\tau_m^k \leqslant 1}{\text{minimize}} & \quad \left(\big\|\nabla f_m(\boldsymbol{\theta}^k)-\boldsymbol{q}_{m}^{k-1}\big\|_2-\left\|\tau_m^k R_m^k \boldsymbol{1}\right\|_2\right)^2 \\
\text { subject to } & \quad \tau_m^k=\cfrac{1}{\left(2^{b_m^k}-1\right)}
\label{optimization_fun}
\end{array}.
\end{equation}

\begin{theorem}
    Solving the optimization problem \eqref{optimization_fun} gives AQUILA an adaptive strategy:
\begin{equation}
(b_m^k)^* = \left\lfloor \log_2 \left(\frac{ R_m^k\sqrt{d}}{\left\|\nabla f_m(\boldsymbol{\theta}^k)-\boldsymbol{q}_{m}^{k-1}\right\|_2} + 1 \right)\right\rfloor,
\label{b_m^k}
\end{equation}
which selects the optimal quantization level based on the quantization range $R_m^k$, the dimension $d$ of the local model, the current gradient $\nabla f_m(\boldsymbol{\theta}^k)$, and the last uploaded quantized gradient $\boldsymbol{q}_{m}^{k-1}$.
\end{theorem}
\begin{proof}
Since $R_m^k$ is independent of $\tau_m^k$, we can formulate an optimization problem about $\tau_m^k$ for device $m$ at communication round $k$ as \eqref{optimization_fun}. Therefore, the optimal solution of $\tau_m^k$ is
\begin{equation}
(\tau_m^k)^*= \cfrac{\big\|\nabla f_m(\boldsymbol{\theta}^k)-\boldsymbol{q}_{m}^{k-1}\big\|_2}{R_m^k \sqrt{d}}.
\label{tau_m^k}
\end{equation}
Then, the optimal adaptive quantization level $(b_m^k)^*$ is equal to
\begin{equation}
\begin{aligned}
(b_m^k)^* &= \left\lfloor \log_2 (\frac{1}{(\tau_m^k)^*} + 1) \right\rfloor \\
&= \bigg\lfloor \log_2 \bigg(\frac{ R_m^k\sqrt{d}}{\big\|\nabla f_m(\boldsymbol{\theta}^k)-\boldsymbol{q}_{m}^{k-1}\big\|_2} + 1 \bigg)\bigg\rfloor.
\end{aligned}
\end{equation}
\end{proof}
\begin{remark}
Unlike certain adaptive quantization algorithms, such as DAdaQuant \cite{honig2022dadaquant}, which necessitate a maximization operation of the computed quantization level results (e.g., $b_m^k=\max (1, \operatorname{round}(\sqrt{\frac{a}{b}} \times w_i^{2 / 3}))$), AQUILA's method of determining the optimal quantization level is self-consistent, because $(b_m^k)^* \geqslant 1 $ is always true since $(\tau_m^k)^*\leqslant 1$.
\end{remark}

\begin{center} 
	\begin{minipage}{1\linewidth}
		\renewcommand{\algorithmicrequire}{\textbf{Input:}}
		\renewcommand{\algorithmicensure}{\textbf{Initialize:}}
		\begin{algorithm}[H]		
            \label{AQUILA_algoirthmm}
		\caption{Communication Efficient FL with AQUILA}
			\begin{algorithmic}[1]
				\Require the number of communication rounds $K$, the learning rate $\alpha$.
				\Ensure the initial global model parameter $\boldsymbol{\theta} ^{0}$.
				\State Server broadcasts $\boldsymbol{\theta} ^{0}$ to all devices. 
				\For {each device $m\in\mathcal{M}$ \textbf{in parallel}} 
				 \State Calculates local gradient ${\nabla f_m(\boldsymbol{\theta}}^{0})$.
				 \State Compute $(b_m^0)^{*}$ by setting $\boldsymbol{q}_m^{k-1} = \boldsymbol{0}$ in \eqref{b_m^k} and the quantized gradient innovation $\Delta \boldsymbol{q}_{m}^{0}$, and transmits it back to the server side.
                \EndFor
				\For {$k = 1,2,...,K$}
				\State Server broadcasts $\boldsymbol{\theta} ^{k}$ to all devices. 
				\For {each device $m\in\mathcal{M}$ \textbf{in parallel}}

                \State Calculates local gradient ${\nabla f_m(\boldsymbol{\theta}}^{k})$, the optimal local quantization level $(b_{m}^{k})^*$ by \eqref{b_m^k}, and the quantized gradient innovation $\Delta \boldsymbol{q}_{m}^{k}$. 
				
				\If {\eqref{skip_rule} does not hold for device $m$} 
				\State device $m$ transmits $\Delta \boldsymbol{q}_{m}^{k}$ to the server.
				\EndIf
				
				\EndFor
				\State Server updates $\boldsymbol{\theta} ^{k+1}$ by the saving previous global quantized gradient $\boldsymbol{q}_m^{k - 1}$ and the received quantized gradient innovation $\Delta \boldsymbol{q}_m^k$: $\boldsymbol{\theta} ^{k+1}:= \boldsymbol{\theta} ^{k}-\alpha \left( \boldsymbol{q}^{k - 1} + 1 / M\sum_{m \in \mathcal{M}^k} \Delta \boldsymbol{q}_m^k \right)$.
				\State Server saves the average quantized gradient $\boldsymbol{q}^k$ for the next aggregation.
				\EndFor
			\end{algorithmic}
		\end{algorithm}
	\end{minipage}
\end{center}

The superiority of \eqref{b_m^k} comes from the following three aspects. First, since $R_m^k \geqslant [\nabla f_m(\boldsymbol{\theta}^k)]_i - [\boldsymbol{q}_{m}^{k-1}]_i \geqslant - R_m^k$, the optimal quantization level $(b_m^k)^*$ must be greater than or equal to $1$. 
Second, AQUILA can personalize an optimal quantization level for each device corresponding to its own gradient, whereas, in AdaQuantFL, each device merely utilizes an identical quantization level according to the global loss.
Third, the gradient innovation and quantization range $R_m^k$ tend to fluctuate along with the training process instead of keeping descending, and thus prevent the quantization level from increasing tremendously compared with AdaQuantFL. 



The detailed process of AQUILA is comprehensively summarized in \Cref{AQUILA_algoirthmm}. At epoch $k=0$, each device calculates $b_m^0$ by setting $\boldsymbol{q}_0^{k - 1} = \boldsymbol{0}$ and uploads $\Delta \boldsymbol{q}_0^k$ to the server since the \eqref{skip_rule} is not satisfied. At epoch $k \in \left\{1, 2, ..., K\right\}$, the server first broadcasts the global model $\boldsymbol{\boldsymbol{\theta}} ^{k}$ to all devices. Each device $m$ computes ${\nabla f (\boldsymbol{\theta}}_{m}^{k})$ with local training data and then utilizes it to calculate an optimal quantization level by \eqref{b_m^k}. Subsequently, each device computes its gradient innovation after quantization and determines whether or not to upload based on the communication criterion \eqref{skip_rule}. Finally, the server updates the new global model $\boldsymbol{\boldsymbol{\theta}} ^{k+1}$ with up-to-date quantized gradients $\boldsymbol{q}_{m}^{k-1} + \Delta \boldsymbol{q}_{m}^{k}$ for those devices who transmit the uploads at epoch $k$, while reusing the old quantized gradients $\boldsymbol{q}_{m}^{k-1}$ for those who skip the uploads.

\section{Theoretical Derivation and Analysis of AQUILA}
\label{development}
As aforementioned, we bound the model deviation caused by skipping updates with respect to quantization bits. Specifically, if the communication criterion \eqref{skip_rule} holds for the device $m$ at epoch $k$, it does not contribute to epoch $k$'s gradient. Otherwise, the loss caused by device $m$ will be minimized with the optimal quantization level selection criterion \eqref{b_m^k}. In this section, the theoretical convergence derivation of AQUILA is based on the following standard assumptions.

\begin{assumption}
\textit{
Each local objective function $f_{m}$ is $L_{m}$-smooth, i.e., there exist a constant $L_m > 0$, such that $\forall \boldsymbol{x}, \boldsymbol{y} \in \mathbb{R}^d$,}
\begin{equation}
\left\|\nabla f_m(\mathbf{x})-\nabla f_m(\mathbf{y})\right\|_2\leqslant L_m\left\| \mathbf{x}-\mathbf{y} \right\|_2,
\end{equation}
which implies that the global objective function $f$ is $L$-smooth with $L \leq \Bar{L} = \frac{1}{m}\sum_{i = 1}^m L_m$.
\end{assumption}

\begin{assumption}(Uniform lower bound).
    \textit{For all $\boldsymbol{x} \in \mathbb{R}^{d}$, there exist $f^{*} \in \mathbb{R}$ such that $f(\boldsymbol{x}) \geq f^*$.}
\end{assumption}

\begin{lemma}
\textit{Following the assumption that the function $f$ is L-smooth, we have}
\begin{equation}
\begin{aligned}
f(\boldsymbol{\theta}^{k+1})-&f(\boldsymbol{\theta}^k)\leqslant-\frac{\alpha}{2}\left\|\nabla f(\boldsymbol{\theta}^k)\right\|_2^2
\\
&+\alpha \left(\bigg\|\frac{1}{M} \sum_{m \in \mathcal{M}_c^k} \Delta \boldsymbol{q}_{m}^{k} \bigg\|_2^2 + \left\|\boldsymbol{\varepsilon}^k\right\|_2^2\right)
\\
&+\left(\frac{L}{2}-\frac{1}{2\alpha}\right)\left\|\boldsymbol{\theta}^{k+1}-\boldsymbol{\theta}^k\right\|_2^2.
\label{delta_aquila}
\end{aligned}
\end{equation}
\end{lemma}

\begin{assumption}
    \label{assumption3}
    \textit{All devices' quantization errors $\boldsymbol{\varepsilon}^k$ will be constrained by the total error of the omitted devices., i.e., $\forall \ k = 0, 1, \cdots, K$, if $\mathcal{M}_c^k \neq \varnothing$, $\exists \ \gamma \geqslant 1$, such that}
\begin{equation}
\left\|\boldsymbol{\varepsilon}^k\right\|_2^2 = \bigg\|\frac{1}{M} \sum_{m \in \mathcal{M}} \boldsymbol{\varepsilon}_m^k\bigg\|_2^2 \leqslant \frac{\gamma }{M^2}\bigg\|\sum_{m \in \mathcal{M}_c^k} \boldsymbol{\varepsilon}_m^k\bigg\|_2^2,
\label{gamma_rule}
\end{equation}
where K denotes the termination time, and $\boldsymbol{\varepsilon}_m^k = \nabla f_m(\boldsymbol{\theta}^{k}) -\left(\boldsymbol{q}_{m}^{k-1} + \Delta \boldsymbol{q}_{m}^{k}\right)$.
\end{assumption}
This assumption is easy to verify when $\mathcal{M}_c^k \neq \varnothing$, a bounded variable (here is $\boldsymbol{\varepsilon}^k$) will always be bounded by a part of itself ($\frac{1}{M}\sum_{m \in \mathcal{M}_c^k} \boldsymbol{\varepsilon}_m^k$) multiplied by a real number ($\gamma$). Note that there is another nontrivial scenario that $\mathcal{M}_c^k \neq \varnothing$ but $\boldsymbol{\varepsilon}_m^k = 0$ for all $m \in \mathcal{M}_c^k$, which implies that $\gamma = 0$ or not exists and conflicts with our assumption. However, this situation only happens when all entries of $\boldsymbol{\varepsilon}_m^k = 0$, i.e., 
$[\nabla f_m(\boldsymbol{\theta}^k)]_i =[\boldsymbol{q}_{m}^{k-1}]_i$ for all $0 \leqslant i \leqslant d$.

\begin{lemma}
\label{Lemma3}
    \textit{The summation of quantized gradient innovation and quantization error is bounded by the global model difference:
\begin{equation}
    \bigg\|\frac{1}{M}\sum_{m \in \mathcal{M}_c^k} \Delta \boldsymbol{q}_{m}^{k}\bigg\|_2^2+\left\| \boldsymbol{\varepsilon}^k \right\|_2^2 \leqslant \frac{\beta\gamma}{\alpha^2}\left\|\boldsymbol{\theta}^{k}-\boldsymbol{\theta}^{k-1}\right\|_2^2,
    \label{error_phi}
\end{equation}}\end{lemma}
\begin{proof}\begin{equation}\begin{aligned}
&\bigg\|\frac{1}{M}\sum_{m \in \mathcal{M}_c^k} \Delta \boldsymbol{q}_{m}^{k}\bigg\|_2^2+\left\| \boldsymbol{\varepsilon}^k \right\|_2^2\\
\overset{(a)}{\leqslant} &\bigg\|\frac{1}{M}\sum_{m \in \mathcal{M}_c^k} \Delta \boldsymbol{q}_{m}^{k}\bigg\|_2^2+\gamma\bigg\|\frac{1}{M}\sum_{m \in \mathcal{M}_c^k} \boldsymbol{\varepsilon}_m^k\bigg\|_2^2 \\
\overset{\eqref{sum_ineq}}{\leqslant} &|\mathcal{M}_c^k|\sum_{m \in \mathcal{M}_c^k}\bigg\|\frac{1}{M} \Delta \boldsymbol{q}_{m}^{k}\bigg\|_2^2+\gamma |\mathcal{M}_c^k| \sum_{m \in \mathcal{M}_c^k}\bigg\| \frac{1}{M}\boldsymbol{\varepsilon}_m^k\bigg\|_2^2 \\
{=} &\frac{|\mathcal{M}_c^k| }{M^2}\sum_{m \in \mathcal{M}_c^k}\left(\left\|\Delta \boldsymbol{q}_{m}^{k}\right\|_2^2+\gamma\left\|\boldsymbol{\varepsilon}_m^k\right\|_2^2\right)\\
\overset{(b)}{\leqslant} & \frac{|\mathcal{M}_c^k| }{M^2}\sum_{m \in \mathcal{M}_c^k}\left(\gamma \left\|\Delta \boldsymbol{q}_{m}^{k}\right\|_2^2+\gamma\left\|\boldsymbol{\varepsilon}_m^k\right\|_2^2\right)\\
\overset{(c)}{\leqslant} & \frac{\beta\gamma|\mathcal{M}_c^k|^2}{\alpha^2 M^2}\left\|\boldsymbol{\theta}^{k}-\boldsymbol{\theta}^{k-1}\right\|_2^2 \\
\leqslant & \frac{\beta\gamma}{\alpha^2}\left\|\boldsymbol{\theta}^{k}-\boldsymbol{\theta}^{k-1}\right\|_2^2,
\end{aligned}
\end{equation}
where (a) follows \cref{assumption3}, (b) follows $\gamma$ is larger than 1 by definition, and (c) uses our novel trigger condition \eqref{skip_rule}.
\end{proof}
\begin{lemma}
\label{lemma4}
From \cref{def1}, we can derive that the relationship between \textbf{quantized} gradient innovation $\Delta \boldsymbol{q}_m^k$ and its quantization representation $\boldsymbol{\psi}_{m}^{k}$ which applies $b_m^k$ bits for each dimension:
\begin{equation}
    \Delta \boldsymbol{q}_m^k = 2 \tau_m^k R_m^k \boldsymbol{\psi}_m^k - R_m^k \boldsymbol{1},
    \label{delta_qmk}
\end{equation}
where $\boldsymbol{1} \in \mathbb{R}^d$ denotes a vector filled with scalar value $1$. 
\end{lemma}
\begin{proof}
    This Lemma can easily be derived by the definition of the deterministic mid-tread quantizer \ref{psi_q}. 
\end{proof}

\begin{remark}
    We can utilize \eqref{delta_qmk} to calculate the quantized gradient innovation in the experimental implementation.
\end{remark}

\subsection{Convergence analysis for the generally non-convex case.} 

\begin{theorem}
\label{theorem1}
Suppose Assumptions 1, 2, and 3 be satisfied. If $\mathcal{M}_c^k \neq \varnothing$, the global objective function $f$ satisfies
\begin{multline}
f(\boldsymbol{\theta}^{k+1})-f(\boldsymbol{\theta}^k)
\overset{}{\leqslant} -\frac{\alpha}{2}\left\|\nabla f(\boldsymbol{\theta}^k)\right\|_2^2  + \left( \frac{L}{2} - \frac{1}{2\alpha} \right) \\\Big\|\boldsymbol{\theta}^{k+1} - \boldsymbol{\theta}^{k} \Big\|_2^2
+ \frac{\beta\gamma}{\alpha} \left\|\boldsymbol{\theta}^{k} - \boldsymbol{\theta}^{k-1}\right\|_2^2,
\label{non_cvx}
\end{multline}
\end{theorem}

\begin{proof}
    Suppose Assumptions 4.1, 4.2, and 4.3 are satisfied and $M_c^k \neq \varnothing$. For the simplicity of the convergence proof, we assume $\Phi^k = \frac{1}{M}\sum_{m \in \mathcal{M}_c^k}\Delta \boldsymbol{q}_{m}^{k}$. First, we prove \cref{lemma4}.
\begin{equation}
\thinmuskip=0mu
\medmuskip=-0.5mu
\thickmuskip=0mu
\begin{aligned}
\phantom{\leqslant}& f(\boldsymbol{\theta}^{k+1})-f(\boldsymbol{\theta}^k) \\
\leqslant &\left\langle\nabla f(\boldsymbol{\theta}^k), \boldsymbol{\theta}^{k+1}-\boldsymbol{\theta}^k\right\rangle+\frac{L}{2}\left\|\boldsymbol{\theta}^{k+1}-\boldsymbol{\theta}^k\right\|_2^2\\
= & \left\langle\nabla f(\boldsymbol{\theta}^k),-\alpha\left(\nabla f(\boldsymbol{\theta}^k)-\boldsymbol{\varepsilon}^k-\Phi^k\right)\right\rangle+\frac{L}{2}\left\|\boldsymbol{\theta}^{k+1}-\boldsymbol{\theta}^k\right\|_2^2\\
=& -\alpha\left\|\nabla f(\boldsymbol{\theta}^k)\right\|_2^2+\alpha \left\langle\nabla f(\boldsymbol{\theta}^k),\boldsymbol{\varepsilon}^k+\Phi^k\right\rangle+\frac{L}{2}\left\|\boldsymbol{\theta}^{k+1}-\boldsymbol{\theta}^k\right\|_2^2 \\
 \overset{\eqref{inner_pro}}{=}& -\alpha\left\|\nabla f(\boldsymbol{\theta}^k)\right\|_2^2+ \frac{\alpha}{2}\left( \left\|\nabla f(\boldsymbol{\theta}^k)\right\|_2^2 + \left\|\boldsymbol{\varepsilon}^k+\Phi^k\right\|_2^2 \right.\\
 &\left.- \frac{1}{\alpha^2} \left\|\boldsymbol{\theta}^{k+1} - \boldsymbol{\theta}^{k}\right\|_2^2 \right) +\frac{L}{2}\left\|\boldsymbol{\theta}^{k+1}-\boldsymbol{\theta}^k\right\|_2^2 \\
\leqslant &  -\frac{\alpha}{2}\left\|\nabla f(\boldsymbol{\theta}^k)\right\|_2^2 + \frac{\alpha}{2}\left\|\boldsymbol{\varepsilon}^k+\Phi^k\right\|_2^2 + \left( \frac{L}{2} - \frac{1}{2\alpha}\right) \left\|\boldsymbol{\theta}^{k+1} - \boldsymbol{\theta}^{k}\right\|_2^2 \\
 \overset{\eqref{sum_ineq}}{\leqslant} & -\frac{\alpha}{2}\left\|\nabla f(\boldsymbol{\theta}^k)\right\|_2^2 + \alpha \left\|\boldsymbol{\varepsilon}^k\right\|_2^2+\alpha\left\|\Phi^k\right\|_2^2 \\
 &+ \left( \frac{L}{2} - \frac{1}{2\alpha}\right) \left\|\boldsymbol{\theta}^{k+1} - \boldsymbol{\theta}^{k}\right\|_2^2.
\end{aligned}
\label{lemma4_1}
\end{equation}
Hence, by \Cref{Lemma3}, it gives us \cref{theorem1}. 
\end{proof}

\begin{corollary}
\label{corollary1}
\textit{Let all the assumptions of \cref{theorem1} hold and $\frac{L}{2}-\frac{1}{2 \alpha}+\frac{\beta\gamma}{\alpha} \leqslant 0$, then the AQUILA requires
\begin{equation}
    K = \mathcal{O}\left(\frac{2 \omega_1}{\alpha \epsilon^2}\right)
\end{equation} 
communication rounds with $\omega_1\!=\!f\left(\boldsymbol{\theta}^1\right)\!-\!f\left(\boldsymbol{\theta}^{*}\right)\!+\! \frac{\beta\gamma}{\alpha} \left\|\boldsymbol{\theta}^{1}\!-\!\boldsymbol{\theta}^{0}\right\|_2^2$ to achieve $\min _{k} \|\nabla f(\boldsymbol{\theta}^{k})\|_2^2 \leqslant \epsilon^2$.}
\end{corollary}

\begin{proof}
    Sum \eqref{non_cvx} up for $k = 1, 2, \cdots, K$, we have
\begin{equation}
\thinmuskip=0mu
\medmuskip=0mu
\thickmuskip=0mu
\begin{aligned}
    &f(\boldsymbol{\theta}^{K+1})-f(\boldsymbol{\theta}^1) \leqslant -\frac{\alpha}{2} \sum_{k=1}^{K}\left\|\nabla f(\boldsymbol{\theta}^k)\right\|_2^2 + \left( \frac{L}{2} - \frac{1}{2\alpha} \right) \left\|\boldsymbol{\theta}^{K+1} - \right.\\
    &\left.\boldsymbol{\theta}^{K}\right\|_2^2 
    + \sum_{k=1}^{K-1} \left( \frac{L}{2} - \frac{1}{2\alpha} + \frac{\beta\gamma}{\alpha}\right) \left\|\boldsymbol{\theta}^{k+1} - \boldsymbol{\theta}^{k}\right\|_2^2 + \frac{\beta\gamma}{\alpha} \left\|\boldsymbol{\theta}^{1} - \boldsymbol{\theta}^{0}\right\|_2^2.
\end{aligned}
\label{descent}
\end{equation}
Notice that inequality \eqref{descent} holds for both $M_c^k \neq \varnothing$ and $M_c^k = \varnothing$. 
Therefore, for $\left(\frac{L}{2}-\frac{1}{2 \alpha}+\frac{\beta\gamma}{\alpha} \right) \leqslant 0$ and all hyperparameters are chosen properly, considering the minimum of $\|\nabla f(\boldsymbol{\theta}^k)\|_2^2$
\begin{equation}
\thinmuskip=0mu
\medmuskip=0mu
\thickmuskip=0mu
\begin{aligned}
\min _{k=1, 2, \cdots, K} &\left\|\nabla f(\boldsymbol{\theta}^{k})\right\|_2^2 \\ \leqslant & \frac{1}{K} \sum_{k=1}^{K}\left\|\nabla f(\boldsymbol{\theta}^k)\right\|_2^2 \\
\overset{\eqref{descent}}{\leqslant} & \frac{2}{\alpha K}\left(f(\boldsymbol{\theta}^{1})-f(\boldsymbol{\theta}^{K}) + \frac{\beta\gamma}{\alpha} \left\|\boldsymbol{\theta}^{1} - \boldsymbol{\theta}^{0}\right\|_2^2 \right).
\end{aligned}
\end{equation}
For $\left(\frac{L}{2}-\frac{1}{2 \alpha}+\frac{\beta\gamma}{\alpha} \right) \leqslant 0$ and all hyperparameters are chosen properly, we have that
\begin{equation}
\thinmuskip=0mu
\medmuskip=0mu
\thickmuskip=0mu
\begin{aligned}
\min _{k=1, 2, \cdots, K} &\left\|\nabla f(\boldsymbol{\theta}^{k})\right\|_2^2 \\ \leqslant & \frac{2}{\alpha K}\left(f(\boldsymbol{\theta}^{1})-f(\boldsymbol{\theta}^{*}) + \frac{\beta\gamma}{\alpha} \left\|\boldsymbol{\theta}^{1} - \boldsymbol{\theta}^{0}\right\|_2^2\right) \leqslant \epsilon^2,
\end{aligned}
\end{equation}
which demonstrates AQUILA requires $K = \mathcal{O}\left(\cfrac{2 \omega_1}{\alpha\epsilon^2}  \right)$ communication round with $\omega_1 = f(\boldsymbol{\theta}^{1})-f(\boldsymbol{\theta}^{*})+\frac{\beta\gamma}{\alpha} \left\|\boldsymbol{\theta}^{1} - \boldsymbol{\theta}^{0}\right\|_2^2$ to achieve $\min _{k=1, 2, \cdots, K} \left\|\nabla f(\boldsymbol{\theta}^{k})\right\|_2^2 \leqslant \epsilon^2$.
\end{proof}
\begin{corollary}
\textit{As a specific case for \cref{corollary1}, we also proof the feasibility of our algorithm in an extreme condition: all devices skip for updating in round $k$, i.e., $M_c^k = \varnothing$.}
\begin{proof}
    Since the skipping subset of devices are the empty set, from \eqref{lazy_aggre}, we have
    \begin{equation}
    \thinmuskip=0mu
    \medmuskip=0mu
    \thickmuskip=0mu
    \begin{aligned}
        \boldsymbol{\theta} ^{k+1} - \boldsymbol{\theta} ^{k} &= -\frac{\alpha}{M}\sum_{m \in \mathcal{M}^k}\left(\boldsymbol{q}_{m}^{k-1} + \Delta \boldsymbol{q}_{m}^{k}\right) - \frac{\alpha}{M} \sum_{m \in \mathcal{M}_{c}^{k}}\boldsymbol{q}_{m}^{k-1} \\ 
        &{=} - \frac{\alpha}{M}\sum_{m \in \mathcal{M}}\left(\boldsymbol{q}_{m}^{k-1} + \Delta \boldsymbol{q}_{m}^{k}\right) \\
        &\overset{\eqref{def_quanti_error}}{=} - \frac{\alpha}{M}\sum_{m \in \mathcal{M}}\left(\nabla f_m (\boldsymbol{\theta}^k) - \boldsymbol{\varepsilon}_m^k \right) \\
        &{=} - \alpha \left(\nabla f (\boldsymbol{\theta}^k) - \boldsymbol{\varepsilon}^k \right).
    \end{aligned}
    \end{equation}
    
    From \eqref{delta_aquila} we have:
    \begin{equation}
    \thinmuskip=0mu
    \medmuskip=-0mu
    \thickmuskip=0mu
    \label{36}
    \begin{aligned}
    &f(\boldsymbol{\theta}^{k+1})-f(\boldsymbol{\theta}^k)\\
    \leqslant&-\frac{\alpha}{2}\left\|\nabla f(\boldsymbol{\theta}^k)\right\|_2^2+\alpha \bigg\|\frac{1}{M}\sum_{m \in \mathcal{M}_c^k}\Delta \boldsymbol{q}_{m}^{k}\bigg\|_2^2\\
    &+\left(\frac{L}{2}-\frac{1}{2\alpha}\right)\left\|\boldsymbol{\theta}^{k+1}-\boldsymbol{\theta}^k\right\|_2^2 +\alpha\left\|\boldsymbol{\varepsilon}^k\right\|_2^2\\
    \leqslant& -\frac{\alpha}{2}\left\|\nabla f(\boldsymbol{\theta}^k)\right\|_2^2+\left(\frac{L}{2}-\frac{1}{2\alpha}\right)\left\|\boldsymbol{\theta}^{k+1}-\boldsymbol{\theta}^k\right\|_2^2+\alpha\left\|\boldsymbol{\varepsilon}^k\right\|_2^2\\
    \overset{\eqref{youngs}}{\leqslant}&-\frac{\alpha}{2}\left\|\nabla f(\boldsymbol{\theta}^k)\right\|_2^2+\alpha^2\left(\frac{L}{2}-\frac{1}{2 \alpha}\right)\bigg((1+p)\left\|\nabla f(\boldsymbol{\theta}^k)\right\|_2^2\\
    &+\left(1+p^{-1}\right) \left\| \boldsymbol{\varepsilon}^k \right\|_2^2\bigg)+\alpha \left\| \boldsymbol{\varepsilon}^k\right\|_2^2\\
    =&-\frac{\alpha}{2}\left\|\nabla f(\boldsymbol{\theta}^k)\right\|_2^2+\frac{1}{2}\left(\alpha^2 L-\alpha\right)(1+p)\left\|\nabla f(\boldsymbol{\theta}^k)\right\|_2^2 \\
    &+\frac{1}{2}\left(\alpha^2 L-\alpha\right)\left(1+p^{-1}\right)\left\|\boldsymbol{\varepsilon}^k\right\|_2^2+\alpha\left\|\boldsymbol{\varepsilon}^k\right\|_2^2\\
    =&\frac{\alpha}{2}\left(\left(\alpha L- 1\right)(1+p)-1\right)\left\|\nabla f(\boldsymbol{\theta}^k)\right\|_2^2 \\
    &+\frac{\alpha}{2}\left((\alpha L-1)\left(1+p^{-1}\right)+2\right)\left\|\boldsymbol{\varepsilon}^k\right\|_2^2.
    \end{aligned}
    \end{equation}
    
    If the factor of $\left\|\boldsymbol{\varepsilon}^k\right\|_2^2$ in \eqref{36} is less than or equal to $0$,
    \begin{equation}
        (\alpha L-1)\left(1+p^{-1}\right)+2 \leqslant 0,
        \label{40}
    \end{equation}
    then the factor of $\|\nabla f(\boldsymbol{\theta}^k)\|_2^2$ will be less than $-\frac{\alpha}{2}$, 
    which indicates that
    \begin{equation}
    f(\boldsymbol{\theta}^{k+1})-f(\boldsymbol{\theta}^k) \leqslant -\frac{\alpha}{2}\left\|\nabla f(\boldsymbol{\theta}^k)\right\|_2^2.
    \end{equation}

    Note that it is not difficult to demonstrate that \eqref{40} and $\frac{L}{2}-\frac{1}{2 \alpha}+\frac{\beta\gamma}{\alpha} \leqslant 0$ can actually be satisfied at the same time. For instance, we can set $p = 0.1, \alpha = 0.1, \beta = 0.25, \gamma = 2, L = 2.5$ that satisfies both of them.
\end{proof}
\end{corollary}

\begin{remark}[Compared to LAG]
    Corresponding to eq. (70) in \cite{chen2018lag}, LAG defines a Lyapunov function $\mathbb{V}^k:=f(\boldsymbol{\theta}^k)-f(\boldsymbol{\theta}^*)+\sum_{d=1}^D \beta_d\|\boldsymbol{\theta}^{k+1-d}-\boldsymbol{\theta}^{k-d}\|_2^2$ and claims that it satisfies
\begin{equation}
\mathbb{V}^{k+1}\!-\!\mathbb{V}^k \!\leq\!\!-\!\left(\frac{\alpha}{2}\!-\!\tilde{c}\left(\alpha, \beta_1\right)(1\!+\!\rho)\! \alpha^2\right)\left\|\nabla f(\boldsymbol{\theta}^k)\right\|_2^2,
\label{LAG_noncvx}
\end{equation}
where $\tilde{c}\left(\alpha, \beta_1\right) = L / 2 - 1 / (2\alpha) + \beta_1$, $\beta_1 = D \xi / (2\alpha\eta)$, $\xi < 1 / D$, and $\rho > 0$. The above result \eqref{LAG_noncvx} indicates that LAG requires
\begin{equation}
K_{LAG} = \mathcal{O}\left(\frac{2 \omega_1}{\left(\alpha - 2\tilde{c}\left(\alpha, \beta_1\right)(1+\rho) \alpha^2\right) \epsilon^2}\right)
\end{equation}
communication rounds to converge. Since the non-negativity of the term $\tilde{c}\left(\alpha, \beta_1\right)(1+\rho) \alpha^2$, we can readily derive that $\alpha < \alpha - 2\tilde{c}\left(\alpha, \beta_1\right)(1+\rho) \alpha^2$,
which demonstrates AQUILA achieves a better convergence rate than LAG with the appropriate selection of $\alpha$.
\end{remark}

\subsection{Convergence analysis under Polyak-Łojasiewicz condition.} 

\begin{assumption}($\mathbf{\mu-}$PŁ condition). 
\label{assumption4}
\textit{Function $f$ satisfies the PL condition with a constant $\mu > 0$, that is,}
\begin{equation}
    \left\|\nabla f(\boldsymbol{\theta}^k)\right\|_2^2 \geqslant 2 \mu (f(\boldsymbol{\theta}^k) - f(\boldsymbol{\theta}^*)).
    \label{PL_condtion}
\end{equation}
\end{assumption}

\begin{theorem}
\label{theorem2}
Suppose Assumptions 4.1, 4.2, and 4.3 be satisfied and $\mathcal{M}_c^k \neq \varnothing$, if the hyperparameters satisfy $\frac{\beta\gamma}{\alpha} \leqslant (1-\alpha\mu)\left(\frac{1}{2\alpha} - \frac{L}{2}\right)$, then the global objective function satisfies
\begin{equation}
\begin{aligned}
f(\boldsymbol{\theta}^{k+1})-f(\boldsymbol{\theta}^k) 
{\leqslant}& -\alpha\mu (f(\boldsymbol{\theta}^k)  - f(\boldsymbol{\theta}^*)) \\
&+ \left( \frac{L}{2} - \frac{1}{2\alpha} \right) \left\|\boldsymbol{\theta}^{k+1} -\boldsymbol{\theta}^{k}\right\|_2^2\\
&+ \frac{\beta\gamma}{\alpha} \left\|\boldsymbol{\theta}^{k} - \boldsymbol{\theta}^{k-1}\right\|_2^2,
\end{aligned}
\label{18}
\end{equation}
and the AQUILA requires 
\begin{equation}
    K = \mathcal{O}\left(- \frac{1}{\log (1 - \alpha \mu)}  \log \frac{\omega_1}{\epsilon}\right)
\end{equation} 
communication round with $\omega_1 = f(\boldsymbol{\theta}^1)-f(\boldsymbol{\theta}^{*})+\left(\frac{1}{2\alpha} - \frac{L}{2}\right) \left\|\boldsymbol{\theta}^{1} - \boldsymbol{\theta}^{0}\right\|_2^2$ to achieve $f(\boldsymbol{\theta}^{K+1})-f(\boldsymbol{\theta}^*)+ (\frac{1}{2\alpha} - \frac{L}{2}) \|\boldsymbol{\theta}^{K+1} - \boldsymbol{\theta}^{K}\|_2^2 \leqslant \epsilon$.
\end{theorem}

\begin{proof}
Based on the intermediate result of \cref{theorem1} and \cref{assumption4} ($\mathbf{\mu-}$PŁ condition), we have
\begin{equation}
\thinmuskip=0mu
\medmuskip=0mu
\thickmuskip=0mu
\begin{aligned}
& f(\boldsymbol{\theta}^{k+1})-f(\boldsymbol{\theta}^k) \\
\leqslant & -\frac{\alpha}{2}\left\|\nabla f(\boldsymbol{\theta}^k)\right\|_2^2 + \left( \frac{L}{2} - \frac{1}{2\alpha} \right) \left\|\boldsymbol{\theta}^{k+1} - \boldsymbol{\theta}^{k}\right\|_2^2 + \frac{\beta\gamma}{\alpha} \left\|\boldsymbol{\theta}^{k} - \boldsymbol{\theta}^{k-1}\right\|_2^2 \\
\overset{\eqref{PL_condtion}}{\leqslant}&  -\alpha\mu (f(\boldsymbol{\theta}^k) - f(\boldsymbol{\theta}^*)) + \left( \frac{L}{2} - \frac{1}{2\alpha} \right) \left\|\boldsymbol{\theta}^{k+1} - \boldsymbol{\theta}^{k}\right\|_2^2 \\
&\qquad \qquad \qquad \qquad \qquad \qquad \qquad  + \frac{\beta\gamma}{\alpha} \left\|\boldsymbol{\theta}^{k} - \boldsymbol{\theta}^{k-1}\right\|_2^2,
\end{aligned}
\end{equation}
which is equivalent to 
\begin{equation}
\thinmuskip=0mu
\medmuskip=0mu
\thickmuskip=0mu
\begin{aligned}
f(\boldsymbol{\theta}^{k+1})-f(\boldsymbol{\theta}^*) 
& \overset{\eqref{PL_condtion}}{\leqslant} (1-\alpha\mu) (f(\boldsymbol{\theta}^k) - f(\boldsymbol{\theta}^*)) \\
&+ \left( \frac{L}{2} - \frac{1}{2\alpha} \right) \left\|\boldsymbol{\theta}^{k+1} - \boldsymbol{\theta}^{k}\right\|_2^2 + \frac{\beta\gamma}{\alpha} \left\|\boldsymbol{\theta}^{k} - \boldsymbol{\theta}^{k-1}\right\|_2^2.
\end{aligned}
\end{equation}
Suppose $\frac{\beta\gamma}{\alpha} \leqslant (1-\alpha\mu)\left(\frac{1}{2\alpha} - \frac{L}{2}\right)$, we can show that
\begin{equation}
\begin{aligned}
&f(\boldsymbol{\theta}^{k+1})-f(\boldsymbol{\theta}^*) + \left(\frac{1}{2\alpha} - \frac{L}{2}\right) \left\|\boldsymbol{\theta}^{k+1} - \boldsymbol{\theta}^{k}\right\|_2^2 \\
\leqslant& (1 -\alpha\mu) \left(f(\boldsymbol{\theta}^k) - f(\boldsymbol{\theta}^*) + \left(\frac{1}{2\alpha} - \frac{L}{2}\right) \left\|\boldsymbol{\theta}^{k} - \boldsymbol{\theta}^{k-1}\right\|_2^2\right).
\end{aligned}
\end{equation}
Therefore, after multiply $k = 1, 2, \cdots, K$, we have 
\begin{equation}
\thinmuskip=0mu
\medmuskip=0mu
\thickmuskip=0mu
\begin{aligned}
\hspace*{\fill} &f(\boldsymbol{\theta}^{K+1})-f(\boldsymbol{\theta}^*) + \left(\frac{1}{2\alpha} - \frac{L}{2}\right) \left\|\boldsymbol{\theta}^{K+1} - \boldsymbol{\theta}^{K}\right\|_2^2 \\
\leqslant& (1 -\alpha\mu)^K \left(f(\boldsymbol{\theta}^1) - f(\boldsymbol{\theta}^*)+ \left(\frac{1}{2\alpha} - \frac{L}{2}\right) \left\|\boldsymbol{\theta}^{1} - \boldsymbol{\theta}^{0}\right\|_2^2\right) \leqslant \epsilon,
\end{aligned}
\end{equation}
which demonstrates that our proposed AQUILA requires $K = \mathcal{O}\left(- \frac{1}{\log (1 - \alpha \mu)}  \log \frac{\omega_1}{\epsilon}\right)$ communication round with $\omega_1 = f(\boldsymbol{\theta}^1)-f(\boldsymbol{\theta}^{*})+\left(\frac{1}{2\alpha} - \frac{L}{2}\right) \|\boldsymbol{\theta}^{1} - \boldsymbol{\theta}^{0}\|_2^2$ to achieve $f(\boldsymbol{\theta}^{K + 1})-f(\boldsymbol{\theta}^*)+ \left(\frac{1}{2\alpha} - \frac{L}{2}\right) \|\boldsymbol{\theta}^{K+1} - \boldsymbol{\theta}^{K}\|_2^2 \leqslant \epsilon$.
\end{proof}

\begin{remark}[Compared to LAG]
According to eq. (50) in \cite{chen2018lag}, we have that
\begin{equation}
{\mathbb{V}^K} \leq \left(1- \alpha\mu + \alpha\mu\sqrt{D \xi}\right)^K {\mathbb{V}^0},
\end{equation}
where $\xi < 1 / D$. Thus, we have that LAG requires
\begin{equation}
K_{LAG} =  \mathcal{O}\left(- \frac{1}{\log (1 - \alpha \mu + \alpha\mu\sqrt{D \xi})}  \log \frac{\omega_1}{\epsilon}\right)
\end{equation}
communication rounds to converge. Compared to \cref{theorem2}, we can derive that $\log (1 - \alpha \mu) < \log (1 - \alpha \mu + \alpha\mu\sqrt{D \xi})$, which indicates that AQUILA has a faster convergence than LAG under the PŁ condition.
\end{remark}

\begin{remark}
    We want to emphasize that LAQ introduces the Lyapunov function into its proof, making it extremely complicated. In addition, LAQ can only guarantee that the final objective function converges to a range of the optimal solution rather than an accurate optimum $f(\boldsymbol{\theta}^*)$. Nevertheless, as discussed in \cref{precise_skip}, we utilize the precise model difference in AQUILA as a surrogate for the global gradient and thus simplify the proof.
\end{remark}

\section{Experiments and Discussion}

\begin{table*}[htbp]
\renewcommand\arraystretch{1.7}
\tabcolsep=0.12cm
\centering
\caption{Numerical numbers of total communication bits in the \textbf{homogeneous} environment. \textit{Acc} denote the text accuracy (\%), \textit{PP} denotes the Perplexity, and \textit{Cost} is the value of total communication bits in the entire training process for all devices.}
\label{homo_table}
\resizebox{0.99\textwidth}{!}{
\begin{tabular}{cc|cc|cc|cc|cc|cc|cc|cc}
\hline
\multicolumn{2}{c|}{Total Comm Bits (GB)} & \multicolumn{2}{c|}{QSGD} & \multicolumn{2}{c|}{AdaQ} & \multicolumn{2}{c|}{LAQ} & \multicolumn{2}{c|}{LAdaQ} & \multicolumn{2}{c|}{LENA} & \multicolumn{2}{c|}{MARINA} & \multicolumn{2}{c}{\textbf{AQUILA}} \\ 
Dataset & Data split & Acc/PP & Cost & Acc/PP & Cost & Acc/PP & Cost & Acc/PP & Cost & Acc/PP & Cost & Acc/PP & Cost & Acc/PP & Cost \\ \hline
\multirow{3}{*}{\texttt{CF-10}} & IID-100 & 69.26 & 156.07 & 69.67 & 226.33 & 69.26 & 153.26 & 70.9 & 226.36 & 69.67 & 160.2 & 69.26 & 162.84 & 70.49 & \textbf{138.35} \\
 & IID & 93.38 & 15.61 & 94.85 & 34.19 & 92.65 & 15.22 & 92.65 & 34.18 & 94.12 & 15.95 & 94.12 & 16.28 & 96.32 & \textbf{4.59} \\
 & Non-IID & 92.65 & 15.61 & 91.91 & 20.39 & 94.85 & 14.48 & 94.85 & 19.86 & 93.38 & 17.64 & 94.12 & 16.28 & 94.12 & \textbf{11.53} \\ \hline
\multirow{3}{*}{\texttt{CF-100}} & IID-100 & 47.4 & 165.55 & 49.4 & 224.02 & 51.6 & 164.11 & 50.4 & 223.64 & 50.8 & 166.87 & 49.4 & 167.71 & 49. & \textbf{142.55} \\
 & IID & 67.65 & 16.56 & 64.71 & 28.68 & 68.38 & 16.28 & 63.97 & 14.41 & 68.38 & 16.63 & 68.38 & 16.77 & 75.74 & \textbf{3.98} \\
 & Non-IID & 83.09 & 8.28 & 83.82 & 14.54 & 81.62 & 8.27 & 80.15 & 14.25 & 84.56 & 9.19 & 80.88 & 8.49 & 79.41 & \textbf{6.12} \\ \hline
\multirow{2}{*}{\texttt{WT-2}} & IID-80 & 3.85 & 470.95 & 4.87 & 711.49 & 5.73 & 513.07 & 4.87 & 710.17 & 4.87 & 341.17 & 5.68 & 338.38 & 3.76 & \textbf{218.59} \\
 & IID & 1.68 & 134.56 & 1.68 & 340.97 & 1.72 & 106.92 & 1.68 & 170.40 & 1.68 & 150.07 & 1.68 & 136.31 & 1.75 & \textbf{71.91} \\ \hline
\end{tabular}%
}
\end{table*}

\begin{table*}[htbp]
\renewcommand\arraystretch{2.0}
\tabcolsep=0.12cm
\centering
\caption{Numerical numbers of total communication bits in the \textbf{heterogeneous} environment. \textit{Acc} denote the text accuracy (\%), \textit{PP} denotes the Perplexity, and \textit{Cost} is the value of total communication bits in the entire training process for all devices.}
\label{hetero_table}
\begin{tabular}{cc|cc|cc|cc|cc|cc|cc|cc}
\hline
\multicolumn{2}{c|}{Total Comm Bits (GB)} & \multicolumn{2}{c|}{QSGD} & \multicolumn{2}{c|}{Ada} & \multicolumn{2}{c|}{LAQ} & \multicolumn{2}{c|}{Ada+LAQ} & \multicolumn{2}{c|}{LENA} & \multicolumn{2}{c|}{MARINA} & \multicolumn{2}{c}{\textbf{AQUILA}} \\ 
Dataset & Data split & Acc/PP & Cost & Acc/PP & Cost & Acc/PP & Cost & Acc/PP & Cost & Acc/PP & Cost & Acc/PP & Cost & Acc/PP & Cost \\ \hline
\multirow{2}{*}{\texttt{CF-10}} & IID & 96.32 & 9.76 & 94.85 & 21.99 & 94.85 & 9.55 & 94.12 & 10.98 & 94.85 & 9.97 & 94.85 & 10.18 & 95.59 & \textbf{2.65} \\
 & Non-IID & 97.06 & 9.76 & 97.06 & 16.15 & 97.79 & 9.25 & 95.59 & 14.67 & 97.06 & 11.19 & 97.06 & 10.18 & 97.79 & \textbf{7.16} \\ \hline
\multirow{2}{*}{\texttt{CF-100}} & IID & 75. & 10.56 & 72.79 & 19.42 & 75. & 10.56 & 75.74 & 9.7 & 77.94 & 10.61 & 73.53 & 10.7 & 83.82 & \textbf{2.51} \\
 & Non-IID & 81.62 & 5.28 & 84.56 & 10.07 & 85.29 & 5.28 & 86.03 & 5.02 & 87.5 & 5.56 & 85.29 & 5.42 & 86.03 & \textbf{3.66} \\ \hline
\texttt{WT-2} & IID & 1.26 & 99.09 & 1.26 & 248.87 & 1.26 & 92.74 & 1.26 & 124.47 & 1.26 & 119.83 & 1.26 & 100.38 & 1.46 & \textbf{53.84} \\ \hline
\end{tabular}
\end{table*}

\subsection{Experiment setup}

In this paper, we evaluate AQUILA on \texttt{CIFAR-10}, \texttt{CIFAR-100} \cite{krizhevsky2009learning}, and \texttt{WikiText-2} dataset \cite{merity2016pointer}, considering IID, Non-IID data scenario, and heterogeneous model architecture (which is also a crucial challenge in FL) simultaneously.

The FL environment is simulated in Python 3.9 with \texttt{PyTorch 11.1} \cite{paszke2019pytorch} implementation. For the diversity of the neural network structures, we train \texttt{ResNet-18} \cite{he2016deep} at \texttt{CIFAR-10} (\texttt{CF-10}) dataset, \texttt{MobileNet-v2} \cite{sandler2018mobilenetv2} at \texttt{CIFAR-100} (\texttt{CF-100}) dataset, and \texttt{Transformer} \cite{vaswani2017attention} at \texttt{WikiText-2} (\texttt{WT-2}) dataset. 

As for the FL system setting, considering the large-scale feature of FL, we validate AQUILA on a large system with $M=100/80$ total devices for \texttt{CIFAR} $/$ \texttt{WikiText-2} dataset. The hyperparameters and additional details of our experiments are revealed in Appendix.C (the supplementary file).

\begin{figure*}[htbp]
    \centering
    \includegraphics[width=0.98\linewidth]{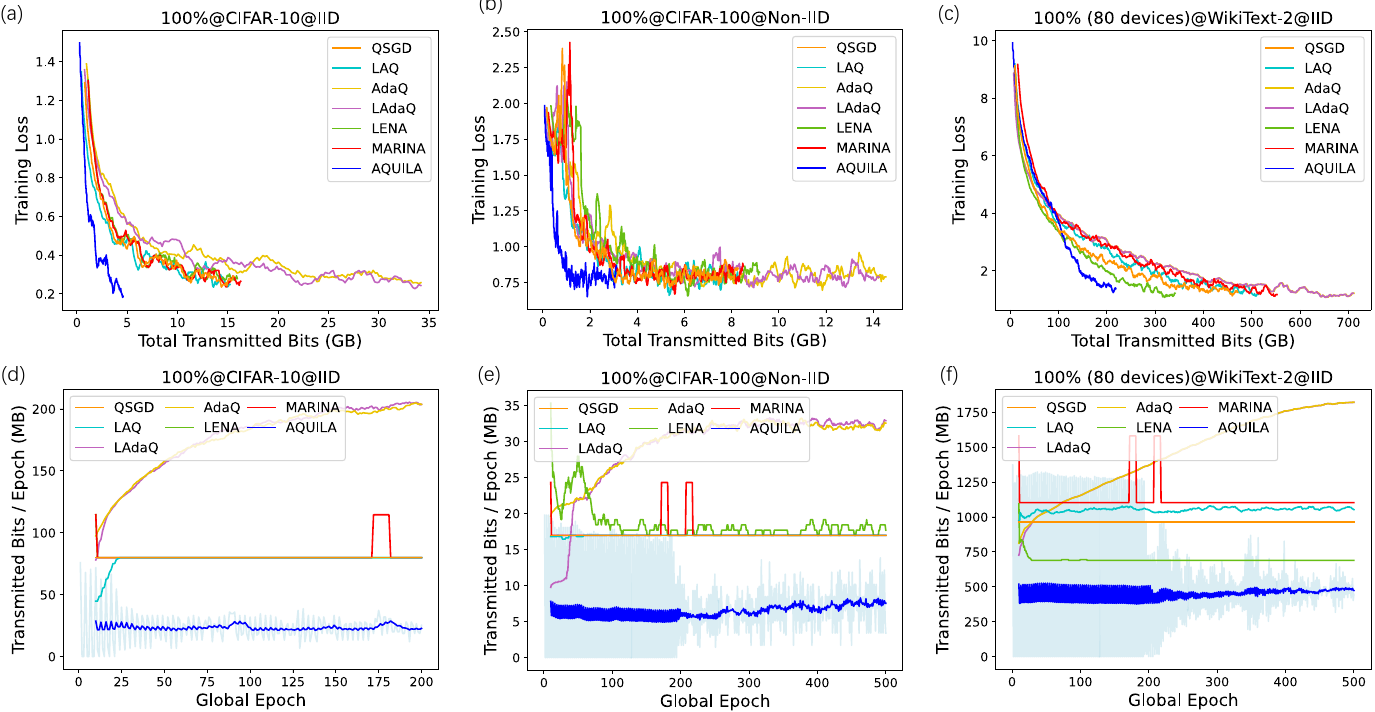}
    \vspace{-8pt}
    \caption{Comparison of AQUILA with other communication-efficient algorithms on IID and Non-IID settings with \textbf{homogeneous} model structure. (a)-(c): training loss v.s. total transmitted bits, (d)-(f): transmitted bits per epoch v.s. global epoch.}
    \label{homo_fig}
\end{figure*}

\begin{figure*}[htbp]
    \centering
    \includegraphics[width=0.98\linewidth]{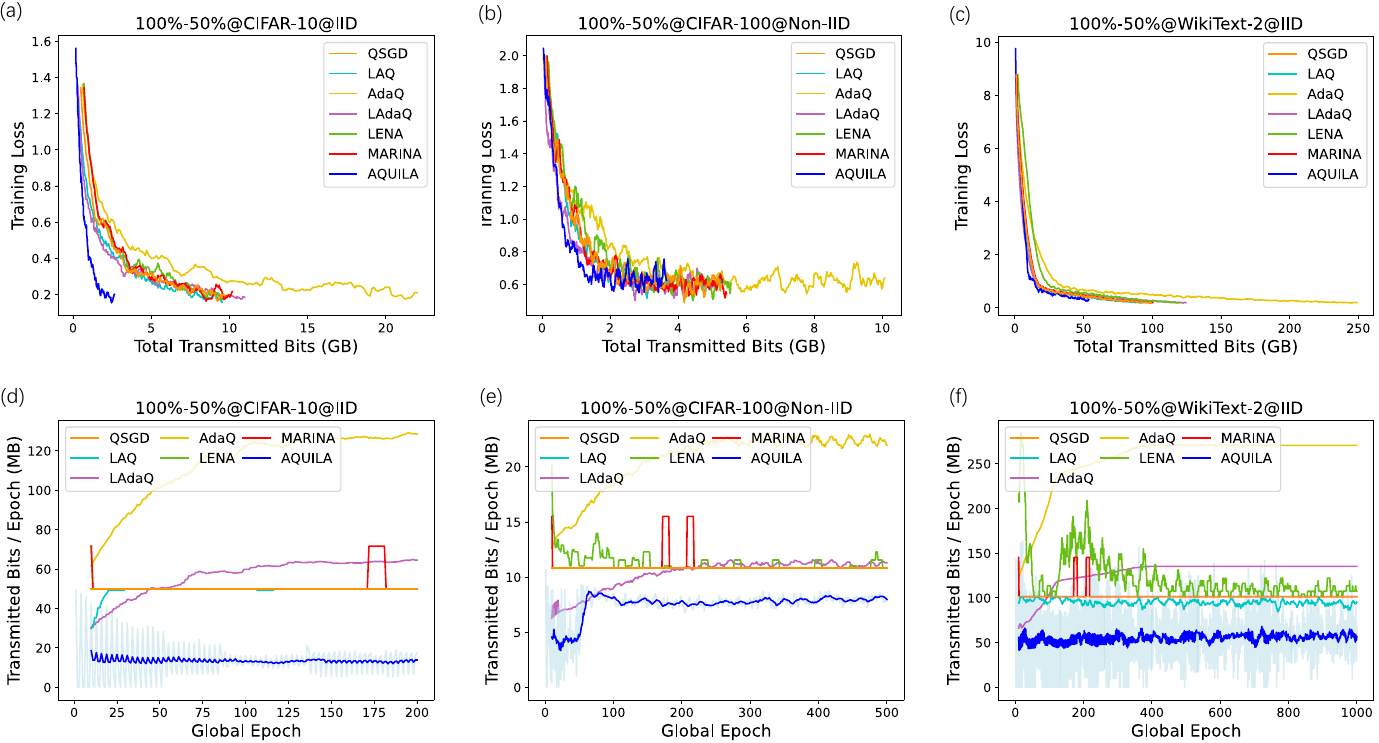}
    \vspace{-8pt}
    \caption{Comparison of AQUILA with other communication-efficient algorithms on IID and Non-IID settings with \textbf{heterogeneous} model structure. (a)-(c): training loss v.s. total transmitted bits, (d)-(f): transmitted bits per epoch v.s. global epoch.}
    \label{hetero_fig}
\end{figure*}

\subsection{Homogeneous environment}

We first evaluate AQUILA with homogeneous settings where all local models share the same model architecture as the global model. To better demonstrate the effectiveness of AQUILA, its performance is compared with several state-of-the-art methods, including AdaQuantFL, LAQ with fixed levels, LENA \cite{ghadikolaei2021lena}, MARINA \cite{gorbunov2021marina}, and the naive combination of AdaQuantFL with LAQ. Note that based on this homogeneous setting, we conduct both IID and Non-IID evaluations on \texttt{CIFAR-10} and \texttt{CIFAR-100} dataset, and an IID evaluation on \texttt{WikiText-2}. To simulate the Non-IID FL setting as \cite{diao2020heterofl}, each device is allocated two classes of data in \texttt{CIFAR-10} and 10 classes of data in \texttt{CIFAR-100} at most, and the amount of data for each label is balanced.

The experimental results are presented in \cref{homo_fig}, where \textit{100\%} implies all local models share a similar structure with the global model (i.e., homogeneity), \textit{100\% (80 devices)} denotes the experiment is conducted in an 80 devices system, and \textit{LAdaQ} represents the naive combination of AdaQuantFL and LAQ. For better illustration, the results have been smoothed by their standard deviation. The solid lines represent values after smoothing, and transparent shades of the same colors around them represent the true values. For the simplicity of the figure, we only display the quantization level change of AQUILA, and the comprehensive experimental results are established in Appendix.C (in a separated file). Additionally, \cref{homo_table} shows the total number of bits transmitted by all devices throughout the FL training process.


\subsection{Non-homogeneous scenario}
In this section, we also evaluate AQUILA with heterogeneous model structures as HeteroFL \cite{diao2020heterofl}, where the structures of local models trained on the device side are heterogeneous. Suppose the global model at epoch $k$ is $\boldsymbol{\theta}^{k}$ and its size is $d = w_g * h_g$, then the local model of each device $m$ can be selected by $\boldsymbol{\theta}_{m}^{k} = \boldsymbol{\theta}^{k}\left[: w_{m},\,: h_{m}\right]$, where $w_{m} = r_{m} w_{g}$ and $h_{m} = r_{m} h_{g}$, respectively. In this paper, we choose model complexity levels $r_m = 0.5$.

Most of the symbols in \cref{hetero_fig} are identical to the \cref{homo_fig}. \textit{100\%-50\%} is a newly introduced symbol that implies half of the devices share the same structure with the global model while another half only have 50\% * 50\% parameters as the global model.

\begin{figure*}[htbp]
    \centering
    \includegraphics[width=0.95\linewidth]{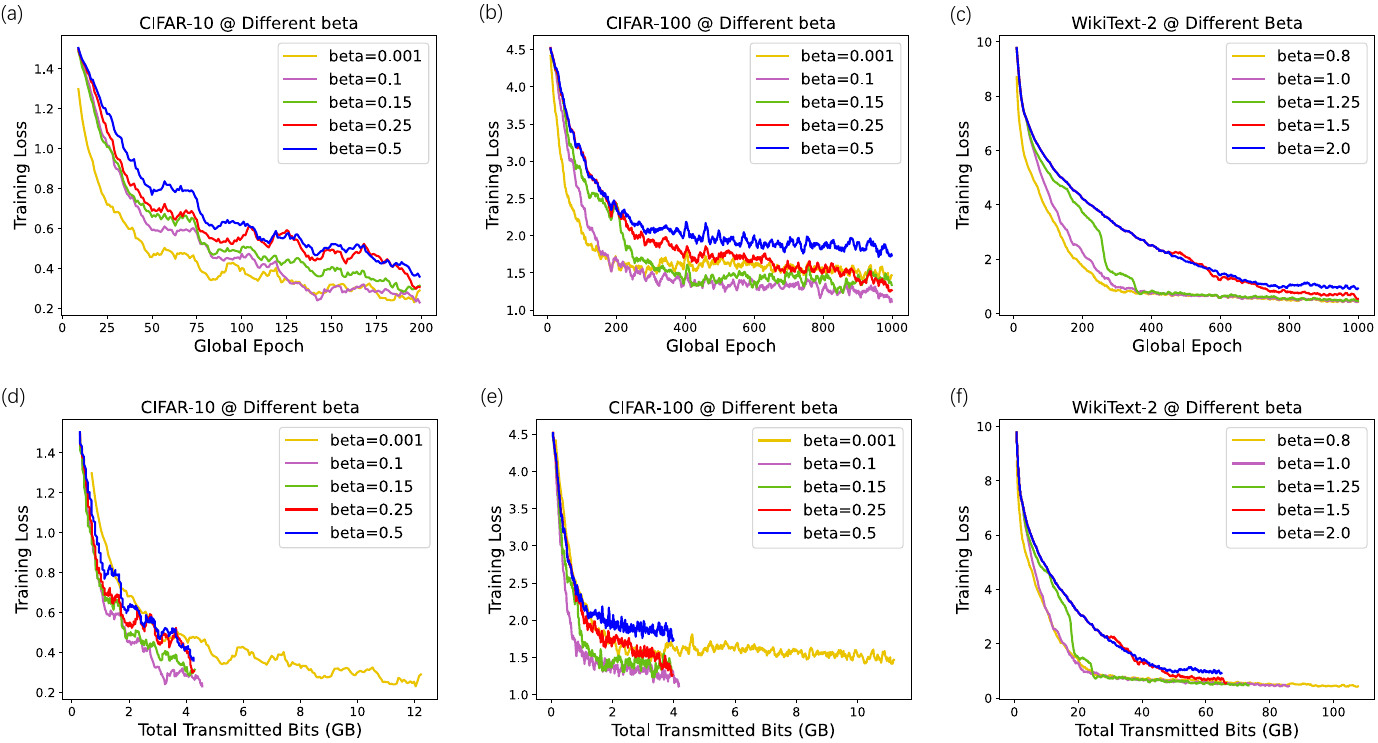}
    \vspace{-8pt}
    \caption{Loss comparison of AQUILA with various selections of the tuning factor $\beta$ in three datasets.}
    \label{diff_beta}
    \vspace{-12pt}
\end{figure*}

\begin{figure*}[htbp]
    \centering
    \includegraphics[width=0.95\linewidth]{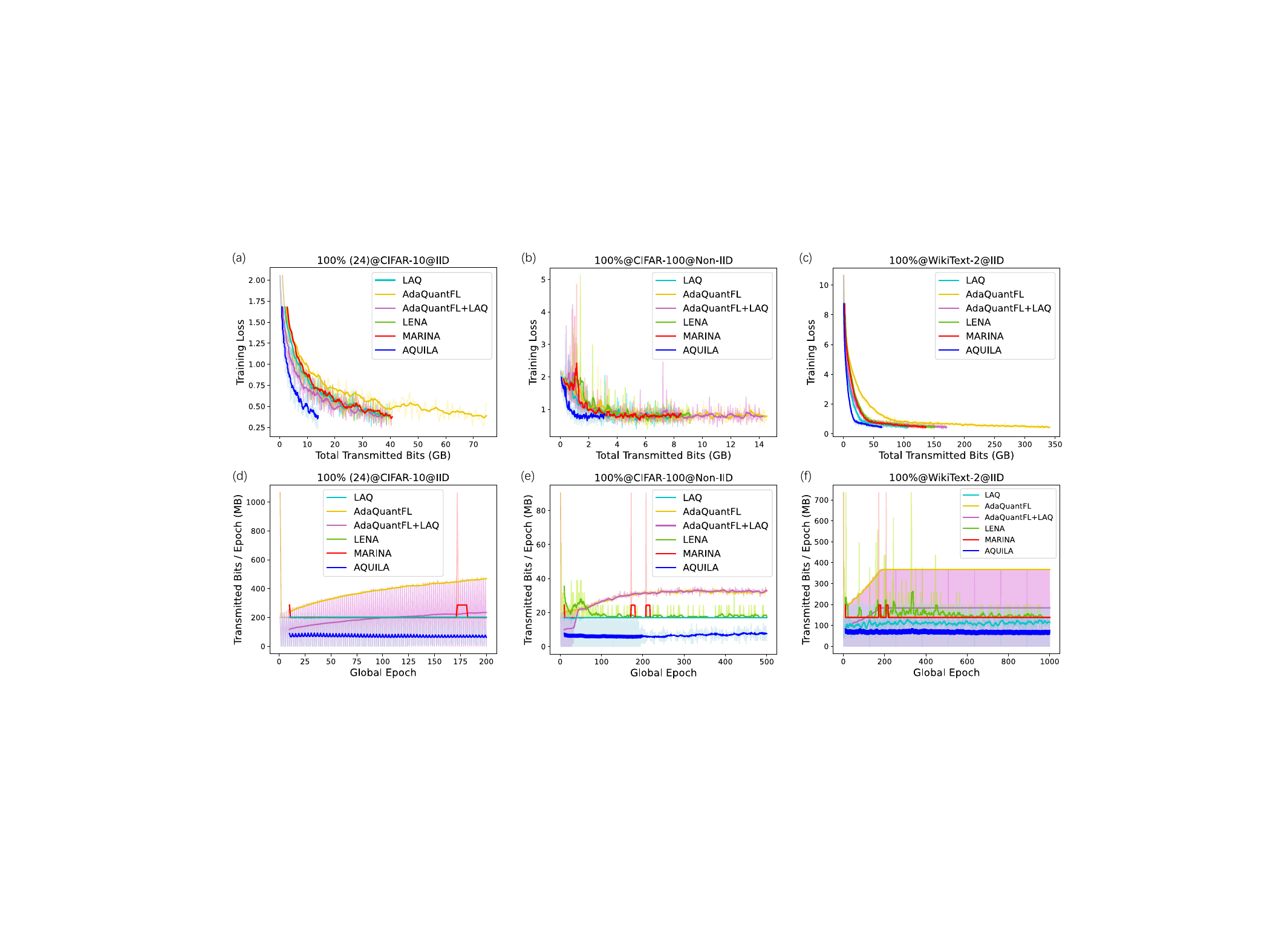}
    \vspace{-8pt}
    \caption{Accuracy (Perplexity) comparison of AQUILA with various selections of the tuning factor $\beta$ in three datasets.}
    \label{beta_acc}
    \vspace{-12pt}
\end{figure*}
\textbf{Performance Analysis.} First of all, AQUILA achieves a significant transmission reduction compared to the naive combination of LAQ and AdaQuantFL in all datasets, which demonstrates the superiority of AQUILA's efficiency. Specifically, \cref{homo_table} indicates that AQUILA saves 57.49\% of transmitted bits in the system of 80 devices at the \texttt{WikiText-2} dataset and reduces 23.08\% of transmitted bits in the system of 100 devices at the \texttt{CIFAR-100} dataset, compared to the naive combination. And other results in \cref{hetero_table} also show an obvious reduction in terms of the total transmitted bits required for convergence. 

Second, in \cref{homo_fig} and \cref{hetero_fig}, the changing trend of AQUILA's communication bits per each round clearly verifies the necessity and effectiveness of our well-designed adaptive quantization level and skip criterion. In these two figures, the number of bits transmitted in each round of AQUILA \textbf{fluctuates} a bit, indicating the effectiveness of AQUILA's selection rule. Meanwhile, the value of transmitted bits remains at quite a low level, suggesting that the adaptive quantization principle makes training more efficient. Moreover, the figures also inform that the quantization level selected by AQUILA will not continuously increase during training instead of being as AdaQuantFL. In addition, based on these two figures, we can also conclude that AQUILA converges faster under the same communication costs.

Finally, AQUILA is capable of adapting to a wide range of challenging FL circumstances. In the Non-IID scenario and heterogeneous model structure, AQUILA still outperforms other algorithms by significantly reducing overall transmitted bits while maintaining the same convergence property and objective function value. In particular, AQUILA reduces 60.4\% overall communication costs compared to LENA and 57.2\% compared to MARINA on average. These experimental results in non-homogeneous FL settings prove that AQUILA can be stably employed in more general and complicated FL scenarios.

\subsection{Ablation study on the impact of tuning factor $\beta$}
One key contribution of AQUILA is presenting a new device selection criterion \eqref{skip_rule} to reduce communication frequency. In this part, we evaluate the effects of the loss performance of different tuning factor $\beta$ value in \cref{diff_beta}. As $\beta$ grows within a certain range, the convergence speed of the model will slow down (due to device skipping). Still, it will eventually converge to the same model performance while considerably reducing the communication overhead. Nevertheless, increasing the value of $\beta$ will lead to a decrease in the final model performance since it skips so many essential uploads that make the training deficient. The accuracy (perplexity) comparison of AQUILA with various selections of the tuning factor $\beta$ is shown in \cref{beta_acc}, which indicates the same trend.To sum up, we should choose the value of factor $\beta$ to maintain the model's performance and minimize the total transmitted amount of bits. Specifically, we select the value of $\beta = 0.1, 0.25, 1.25$ for \texttt{CIFAR-10}, \texttt{CIFAR-100}, and \texttt{WikiText-2} datasets for our evaluation, respectively.

\section{Conclusions and Future Work}
This paper proposes AQUILA, an innovative strategy for adaptive quantization level selection and device selection in FL scenarios. 
Leveraging a novel combination of these strategies, AQUILA has been demonstrated to be capable of reducing the transmitted costs while maintaining the convergence guarantee and model performance compared to existing methods. The evaluation with Non-IID data distribution and various heterogeneous model architectures demonstrates that AQUILA is compatible in a non-homogeneous FL environment.

\section*{Acknowledgement}
This work was supported by the National Key R\&D Program of China under Grant No.2022ZD0160504, by Tsinghua Shenzhen International Graduate School-Shenzhen Pengrui Young Faculty Program of Shenzhen Pengrui Foundation (No. SZPR2023005), and by Tsinghua-Toyota Joint Research Institute inter-disciplinary Program and Tsinghua University (AIR)-Asiainfo Technologies (China) Inc. Joint Research Center under grant No. 20203910074. We would also like to thank anonymous reviewers for their insightful comments.


%

\bibliography{example_paper}

\begin{thebibliography}{10}
\providecommand{\url}[1]{#1}
\csname url@samestyle\endcsname
\providecommand{\newblock}{\relax}
\providecommand{\bibinfo}[2]{#2}
\providecommand{\BIBentrySTDinterwordspacing}{\spaceskip=0pt\relax}
\providecommand{\BIBentryALTinterwordstretchfactor}{4}
\providecommand{\BIBentryALTinterwordspacing}{\spaceskip=\fontdimen2\font plus
\BIBentryALTinterwordstretchfactor\fontdimen3\font minus \fontdimen4\font\relax}
\providecommand{\BIBforeignlanguage}[2]{{%
\expandafter\ifx\csname l@#1\endcsname\relax
\typeout{** WARNING: IEEEtran.bst: No hyphenation pattern has been}%
\typeout{** loaded for the language `#1'. Using the pattern for}%
\typeout{** the default language instead.}%
\else
\language=\csname l@#1\endcsname
\fi
#2}}
\providecommand{\BIBdecl}{\relax}
\BIBdecl

\bibitem{vehicularIoT}
Z.~Du, C.~Wu, T.~Yoshinaga, K.-L.~A. Yau, Y.~Ji, and J.~Li, ``Federated learning for vehicular {Internet} of things: Recent advances and open issues,'' \emph{IEEE Computer Graphics and Applications}, pp. 45--61, 2020.

\bibitem{liu2020fedvision}
Y.~Liu, A.~Huang, Y.~Luo, H.~Huang, Y.~Liu, Y.~Chen, L.~Feng, T.~Chen, H.~Yu, and Q.~Yang, ``Fedvision: An online visual object detection platform powered by federated learning,'' in \emph{Proceedings of the 34th AAAI Conference on Artificial Intelligence}, 2020, pp. 13\,172--13\,179.

\bibitem{Googlekeyboard}
A.~Hard, K.~Rao, R.~Mathews, S.~Ramaswamy, F.~Beaufays, S.~Augenstein, H.~Eichner, C.~Kiddon, and D.~Ramage, ``Federated learning for mobile keyboard prediction,'' \emph{arXiv preprint arXiv:1811.03604}, 2018.

\bibitem{mcmahan2017communication}
B.~McMahan, E.~Moore, D.~Ramage, S.~Hampson, and B.~A. y~Arcas, ``Communication-efficient learning of deep networks from decentralized data,'' in \emph{Artificial Intelligence and Statistics}, 2017, pp. 1273--1282.

\bibitem{sun2020lazily}
J.~Sun, T.~Chen, G.~B. Giannakis, Q.~Yang, and Z.~Yang, ``Lazily aggregated quantized gradient innovation for communication-efficient federated learning,'' \emph{IEEE Transactions on Pattern Analysis \& Machine Intelligence}, pp. 1--15, 2020.

\bibitem{mao2021communication}
Y.~Mao, Z.~Zhao, G.~Yan, Y.~Liu, T.~Lan, L.~Song, and W.~Ding, ``Communication efficient federated learning with adaptive quantization,'' \emph{arXiv preprint arXiv:2104.06023}, 2021.

\bibitem{jhunjhunwala2021adaptive}
D.~Jhunjhunwala, A.~Gadhikar, G.~Joshi, and Y.~C. Eldar, ``Adaptive quantization of model updates for communication-efficient federated learning,'' in \emph{Proceedings of the 2021 IEEE International Conference on Acoustics, Speech and Signal Processing}, 2021, pp. 3110--3114.

\bibitem{honig2022dadaquant}
R.~H{\"o}nig, Y.~Zhao, and R.~Mullins, ``Dadaquant: Doubly-adaptive quantization for communication-efficient federated learning,'' in \emph{International Conference on Machine Learning}.\hskip 1em plus 0.5em minus 0.4em\relax PMLR, 2022, pp. 8852--8866.

\bibitem{cho2020client}
Y.~J. Cho, J.~Wang, and G.~Joshi, ``Client selection in federated learning: Convergence analysis and power-of-choice selection strategies,'' \emph{arXiv preprint arXiv:2010.01243}, 2020.

\bibitem{qu2022feddq}
L.~Qu, S.~Song, and C.-Y. Tsui, ``Feddq: Communication-efficient federated learning with descending quantization,'' in \emph{GLOBECOM 2022-2022 IEEE Global Communications Conference}.\hskip 1em plus 0.5em minus 0.4em\relax IEEE, 2022, pp. 281--286.

\bibitem{lin2021channel}
X.~Lin, Y.~Liu, and F.~Chen, ``Channel-adaptive quantization for wireless federated learning,'' in \emph{2021 IEEE/CIC International Conference on Communications in China (ICCC)}.\hskip 1em plus 0.5em minus 0.4em\relax IEEE, 2021, pp. 457--462.

\bibitem{liu2023communication}
H.~Liu, F.~He, and G.~Cao, ``Communication-efficient federated learning for heterogeneous edge devices based on adaptive gradient quantization,'' in \emph{IEEE INFOCOM 2023-IEEE Conference on Computer Communications}.\hskip 1em plus 0.5em minus 0.4em\relax IEEE, 2023, pp. 1--10.

\bibitem{liu2022ensemble}
Y.-J. Liu, G.~Feng, D.~Niyato, S.~Qin, J.~Zhou, X.~Li, and X.~Xu, ``Ensemble distillation based adaptive quantization for supporting federated learning in wireless networks,'' \emph{IEEE Transactions on Wireless Communications}, 2022.

\bibitem{qu2020quantization}
X.~Qu, J.~Wang, and J.~Xiao, ``Quantization and knowledge distillation for efficient federated learning on edge devices,'' in \emph{2020 IEEE 22nd International Conference on High Performance Computing and Communications; IEEE 18th International Conference on Smart City; IEEE 6th International Conference on Data Science and Systems (HPCC/SmartCity/DSS)}.\hskip 1em plus 0.5em minus 0.4em\relax IEEE, 2020, pp. 967--972.

\bibitem{li2023adaptive}
T.~Li, C.~Yang, L.~Wang, T.~Li, H.~Zhao, and J.~Chen, ``Adaptive quantization mechanism for federated learning models based on dag blockchain,'' \emph{Electronics}, vol.~12, no.~17, p. 3712, 2023.

\bibitem{sun2020adaptive}
H.~Sun, X.~Ma, and R.~Q. Hu, ``Adaptive federated learning with gradient compression in uplink noma,'' \emph{IEEE Transactions on Vehicular Technology}, vol.~69, no.~12, pp. 16\,325--16\,329, 2020.

\bibitem{gholami2022survey}
A.~Gholami, S.~Kim, Z.~Dong, Z.~Yao, M.~W. Mahoney, and K.~Keutzer, ``A survey of quantization methods for efficient neural network inference,'' in \emph{Low-Power Computer Vision}.\hskip 1em plus 0.5em minus 0.4em\relax Chapman and Hall/CRC, 2022, pp. 291--326.

\bibitem{chen2018lag}
T.~Chen, G.~B. Giannakis, T.~Sun, and W.~Yin, ``{LAG}: Lazily aggregated gradient for communication-efficient distributed learning,'' in \emph{Proceedings of Advances in Neural Information Processing Systems}, 2018, pp. 1--25.

\bibitem{krizhevsky2009learning}
\BIBentryALTinterwordspacing
A.~Krizhevsky, G.~Hinton \emph{et~al.}, ``Learning multiple layers of features from tiny images,'' 2009. [Online]. Available: \url{https://www.cs.toronto.edu/~kriz/cifar.html}
\BIBentrySTDinterwordspacing

\bibitem{merity2016pointer}
\BIBentryALTinterwordspacing
S.~Merity, C.~Xiong, J.~Bradbury, and R.~Socher, ``Pointer sentinel mixture models,'' \emph{arXiv preprint arXiv:1609.07843}, 2016. [Online]. Available: \url{https://blog.salesforceairesearch.com/the-wikitext-long-term-dependency-language-modeling-dataset/]}
\BIBentrySTDinterwordspacing

\bibitem{paszke2019pytorch}
A.~Paszke, S.~Gross, F.~Massa, A.~Lerer, J.~Bradbury, G.~Chanan, T.~Killeen, Z.~Lin, N.~Gimelshein, L.~Antiga \emph{et~al.}, ``Pytorch: An imperative style, high-performance deep learning library,'' \emph{Advances in neural information processing systems}, vol.~32, 2019.

\bibitem{he2016deep}
K.~He, X.~Zhang, S.~Ren, and J.~Sun, ``Deep residual learning for image recognition,'' in \emph{Proceedings of the 29th IEEE Conference on Computer Vision and Pattern Recognition}, 2016, pp. 770--778.

\bibitem{sandler2018mobilenetv2}
M.~Sandler, A.~Howard, M.~Zhu, A.~Zhmoginov, and L.-C. Chen, ``Mobilenetv2: Inverted residuals and linear bottlenecks,'' in \emph{Proceedings of the IEEE conference on computer vision and pattern recognition}, 2018, pp. 4510--4520.

\bibitem{vaswani2017attention}
A.~Vaswani, N.~Shazeer, N.~Parmar, J.~Uszkoreit, L.~Jones, A.~N. Gomez, {\L}.~Kaiser, and I.~Polosukhin, ``Attention is all you need,'' \emph{Advances in neural information processing systems}, vol.~30, 2017.

\bibitem{ghadikolaei2021lena}
H.~S. Ghadikolaei, S.~Stich, and M.~Jaggi, ``{LENA}: Communication-efficient distributed learning with self-triggered gradient uploads,'' in \emph{International Conference on Artificial Intelligence and Statistics}.\hskip 1em plus 0.5em minus 0.4em\relax PMLR, 2021, pp. 3943--3951.

\bibitem{gorbunov2021marina}
E.~Gorbunov, K.~P. Burlachenko, Z.~Li, and P.~Richt{\'a}rik, ``{MARINA}: Faster non-convex distributed learning with compression,'' in \emph{International Conference on Machine Learning}.\hskip 1em plus 0.5em minus 0.4em\relax PMLR, 2021, pp. 3788--3798.

\bibitem{diao2020heterofl}
E.~Diao, J.~Ding, and V.~Tarokh, ``{HeteroFL}: Computation and communication efficient federated learning for heterogeneous clients,'' in \emph{Proceedings of the 8th International Conference on Learning Representations}, 2020.

\end{thebibliography}
\bibliographystyle{IEEEtran}

\begin{IEEEbiography}
 [{\includegraphics[width=1in,height=1.25in, clip,keepaspectratio]{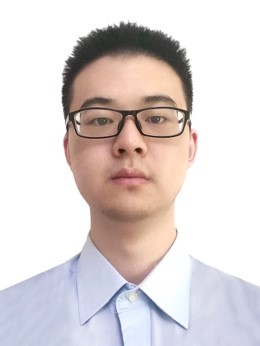}}]{Zihao Zhao}
 received his B.S. degree in University of Electronic Science and Technology of China (UESTC) in 2021. He is currently pursuing his M.S. degree in Data Science and Information Technology at Smart Sensing and Robotics (SSR) group, Tsinghua University. His research interests include Internet of Things (IoTs), Federated Learning, and Machine Learning. 
\end{IEEEbiography}
\vspace{-20pt}
\begin{IEEEbiography}[{\includegraphics[width=1in,height=1.25in, clip,keepaspectratio]{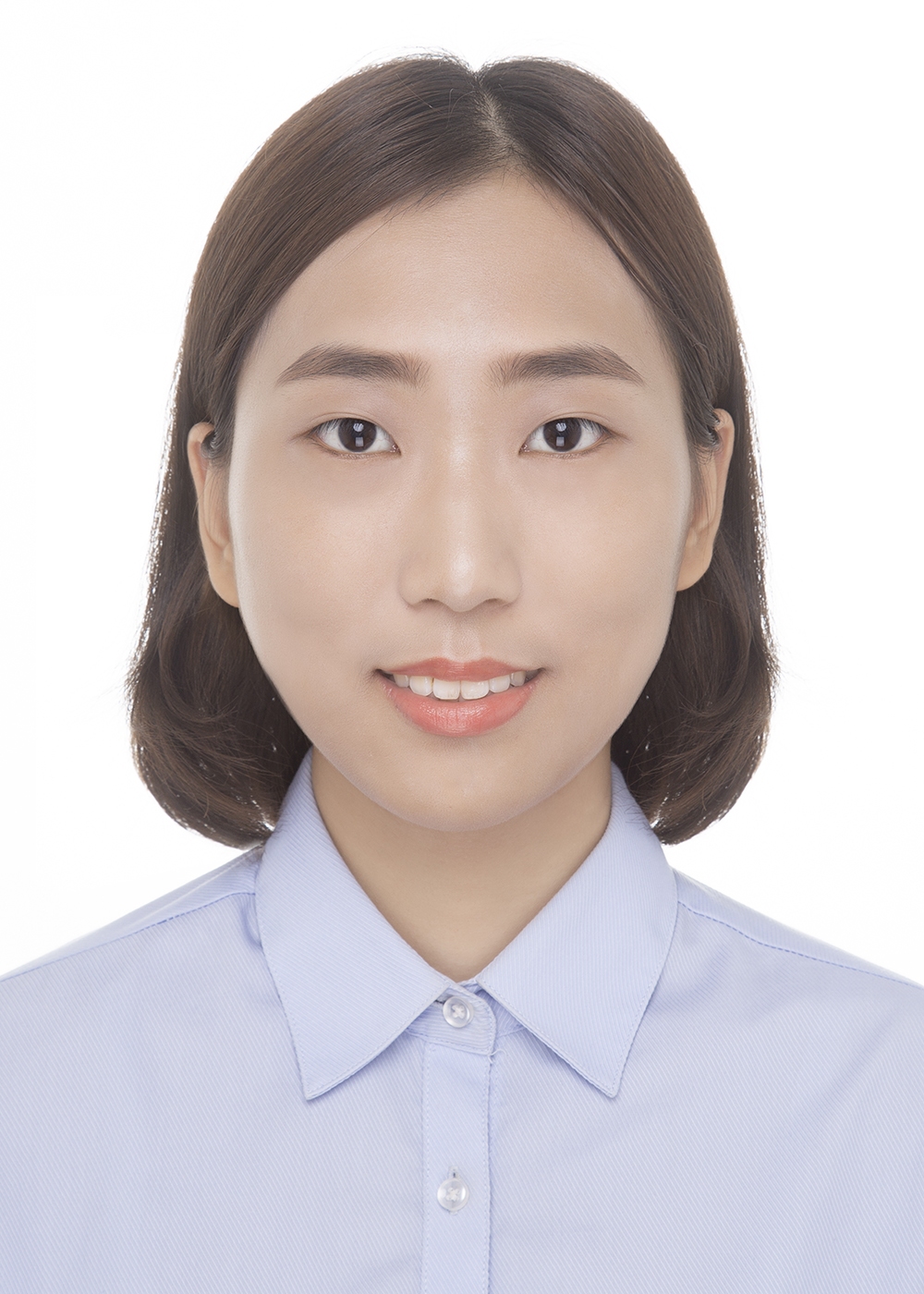}}]
{Yuzhu Mao} received the B.E. degree in computer science from Wuhan University, Wuhan, China, in 2020. Yuzhu Mao is currently pursuing her M.S. degree in Data Science and Information Technology at Smart Sensing and Robotics (SSR) group, Tsinghua University. Her research interests include Federated Learning, Internet of Things (IoTs), and Multi-agent Systems. 
\end{IEEEbiography}
\vspace{-20pt}
\begin{IEEEbiography}
 [{\includegraphics[width=1in, height=1.25in, clip,keepaspectratio]{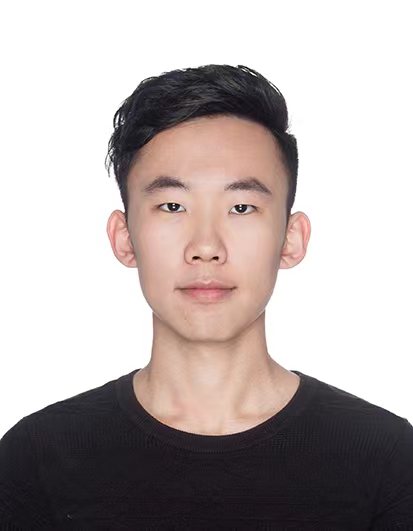}}]{Zhenpeng Shi}
received the B.S. degree in Statistics from Wuhan University in 2021. Zhenpeng Shi currently pursuing his M.S. degree in Data Science and Information Technology at Smart Sensing and Robotics(SSR) group, Tsinghua University. His research interest include Federated Learning, Reinforcement Learning and Multi-agent Systems.
\end{IEEEbiography}
\vspace{-20pt}
\begin{IEEEbiography}
 [{\includegraphics[width=1in,height=1.25in, clip,keepaspectratio]{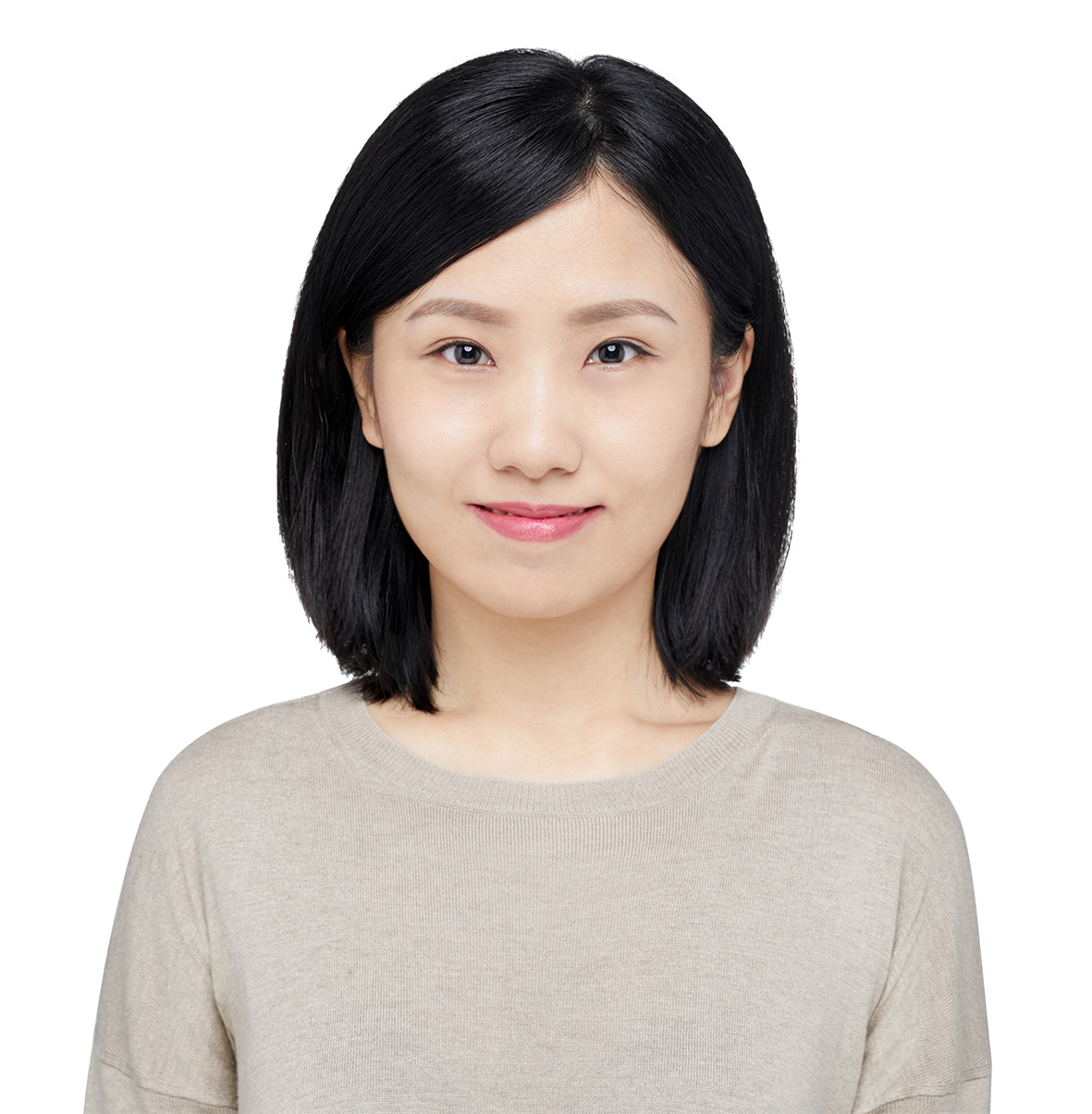}}]{Yang Liu}
is an associate professor with the Institute for AI Industry Research, Tsinghua University. Before joining Tsinghua, she was the principal researcher and research team lead at WeBank. Her research interests include federated learning, machine learning, multi-agent systems, statistical mechanics and AI industrial applications. Her research work was recognized with multiple awards, such as AAAI Innovation Award and CCF Technology Award. She is also named as Innovators on Privacy-Preserving Computation by MIT Technology Review China.
\end{IEEEbiography}
\vspace{-20pt}
\begin{IEEEbiography}
 [{\includegraphics[width=1in,height=1.25in, clip,keepaspectratio]{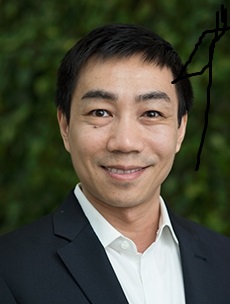}}]{Tian Lan}
received the B.A.Sc. degree from the Tsinghua University, China in 2003, the M.A.Sc. degree from the University of Toronto, Canada, in 2005, and the Ph.D. degree from the Princeton University in 2010. Dr. Lan is currently a full Professor of Electrical and Computer Engineering at the George Washington University. His research interests include network optimization, algorithms, and machine learning. He received the Meta Research Award in 2021, SecureComm Best Paper Award in 2019, SEAS Faculty Recognition Award in 2018, Hegarty Faculty Innovation Award in 2017, AT\&T VURI Award in 2015, IEEE INFOCOM Best Paper Award in 2012,  Wu Prizes for Excellence at Princeton University in 2010, IEEE GLOBECOM Best Paper Award in 2009, and IEEE Signal Processing Society Best Paper Award in 2008.
\end{IEEEbiography}
\vspace{-20pt}
\begin{IEEEbiography}
 [{\includegraphics[width=1in, height=1.25in, clip,keepaspectratio]{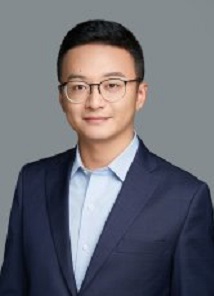}}]{Wenbo Ding}
received the BS and PhD degrees (Hons.) from Tsinghua University in 2011 and 2016, respectively. He worked as a postdoctoral research fellow at Georgia Tech under the supervision of Professor Z. L. Wang from 2016 to 2019. He is now an associate professor and PhD supervisor at Tsinghua-Berkeley Shenzhen Institute, Tsinghua Shenzhen International Graduate School, Tsinghua University, where he leads the Smart Sensing and Robotics (SSR) group. His research interests are diverse and interdisciplinary, which include self-powered sensors, human-machine interfaces, wearable devices for health and robotics with the help of signal processing, machine learning, and mobile computing. He has received many prestigious awards, including the Gold Medal of the 47th International Exhibition of Inventions Geneva and the IEEE Scott Helt Memorial Award. 
\end{IEEEbiography}
\vspace{-20pt}
\begin{IEEEbiography}
[{\includegraphics[width=1in,height=1.25in, clip,keepaspectratio]{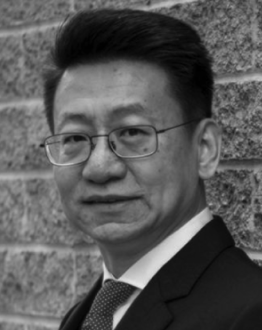}}] {Xiao-Ping Zhang}
received B.S. and Ph.D. degrees from Tsinghua University, in 1992 and 1996, respectively, both in Electronic Engineering. He holds an MBA in Finance, Economics and Entrepreneurship with Honors from the University of Chicago Booth School of Business, Chicago, IL. 

He is a Professor with Tsinghua-Berkeley Shenzhen Institute, Tsinghua University, and with the Department of Electrical, Computer and Biomedical Engineering, Toronto Metropolitan University (Formerly Ryerson University), Toronto, ON, Canada, where he is the Director of the Communication and Signal Processing Applications Laboratory. He is cross-appointed to the Finance Department at the Ted Rogers School of Management, Toronto Metropolitan University. He was a Visiting Scientist with the Research Laboratory of Electronics, Massachusetts Institute of Technology. His research interests include sensor networks and IoT, image and multimedia content analysis, machine learning, statistical signal processing, and applications in big data, finance, and marketing. 

Dr. Zhang is Fellow of the Canadian Academy of Engineering, Fellow of the Engineering Institute of Canada, Fellow of the IEEE, a registered Professional Engineer in Ontario, Canada, and a member of Beta Gamma Sigma Honor Society. He is the general Co-Chair for the IEEE International Conference on Acoustics, Speech, and Signal Processing, 2021. He is the general co-chair for 2017 and 2019 GlobalSIP Symposium on Signal, Information Processing and AI for Finance and Business. He was an elected Member of the ICME steering committee. He is the General Chair for the IEEE International Workshop on Multimedia Signal Processing, 2015. He is Editor-in-Chief for the IEEE JOURNAL OF SELECTED TOPICS IN SIGNAL PROCESSING. He is Senior Area Editor for the IEEE TRANSACTIONS ON IMAGE PROCESSING. He served as Senior Area Editor the IEEE TRANSACTIONS ON SIGNAL PROCESSING and Associate Editor for the IEEE TRANSACTIONS ON IMAGE PROCESSING, the IEEE TRANSACTIONS ON MULTIMEDIA, the IEEE TRANSACTIONS ON CIRCUITS AND SYSTEMS FOR VIDEO TECHNOLOGY, the IEEE TRANSACTIONS ON SIGNAL PROCESSING, and the IEEE SIGNAL PROCESSING LETTERS. He is the Chair of the IEEE Signal Processing Society Technical Committee on Image, Video, and Multidimensional Signal Processing (IVMSP). He received Sarwan Sahota Ryerson Distinguished Scholar Award -- the Ryerson University highest honor for scholarly, research and creative achievements. He is an IEEE Distinguished Lecturer of the IEEE Signal Processing Society, and of the IEEE Circuits and Systems Society. 
\end{IEEEbiography}
\vspace{-20pt}

\vfill

\end{document}